\documentclass[letterpaper]{article} 
\usepackage{aaai23}  
\usepackage{times}  
\usepackage{helvet}  
\usepackage{courier}  
\usepackage[hyphens]{url}  
\usepackage{graphicx} 
\usepackage{natbib}  
\usepackage{caption} 
\usepackage{algorithm}
\usepackage{algorithmic}

\usepackage{amsmath}
\usepackage{amsthm}
\usepackage{bm}
\usepackage{bbold}
\usepackage{amssymb}
\usepackage{mathrsfs}
\usepackage{subfigure}
\usepackage{booktabs}
\usepackage[usenames,dvipsnames]{color}
\usepackage{colortbl}
\usepackage{multirow}
\definecolor{mygray}{gray}{0.9}
\definecolor{mygray1}{gray}{0.95}

\newtheorem{thm}{Theorem}
\newtheorem{define}{Definition}

\newtheorem{lemma}[thm]{Lemma}
\newtheorem{proposition}{Proposition}

\newcommand{\ie}{\textit{i}.\textit{e}.}

\usepackage{newfloat}
\usepackage{listings}
\DeclareCaptionStyle{ruled}{labelfont=normalfont,labelsep=colon,strut=off} 
\floatstyle{ruled}
\newfloat{listing}{tb}{lst}{}
\floatname{listing}{Listing}
\title{Prototypical Partial Optimal Transport for Universal Domain Adaptation}
\author {
	Yucheng Yang\equalcontrib,
	Xiang Gu\equalcontrib, 
	Jian Sun \thanks{Corresponding Author.}
}
\affiliations {
	School of Mathematics and Statistics, Xi’an Jiaotong University, Xi’an, China\\
	ycyang@stu.xjtu.edu.cn, \{xianggu, jiansun\}@xjtu.edu.cn
}

\begin{document}
	
	\maketitle
	
	\begin{abstract}
	Universal domain adaptation (UniDA) aims to transfer knowledge from a labeled source domain to an unlabeled target domain without requiring the same label sets of both domains. The existence of domain and category shift makes the task challenging and requires us to distinguish “known” samples (\ie, samples whose labels exist in both domains) and “unknown” samples (\ie, samples whose labels exist in only one domain) in both domains before reducing the domain gap. In this paper, we consider the problem from the point of view of distribution matching which we only need to align two distributions partially. A novel approach, dubbed mini-batch Prototypical Partial Optimal Transport (m-PPOT), is proposed to conduct partial distribution alignment for UniDA. In training phase, besides minimizing m-PPOT, we also leverage the transport plan of m-PPOT to reweight source prototypes and target samples, and design reweighted entropy loss and reweighted cross-entropy loss to distinguish “known” and “unknown” samples. Experiments on four benchmarks show that our method outperforms the previous state-of-the-art UniDA methods. 
	\end{abstract}
	
	\section{Introduction}
	Deep Learning has achieved significant progress in image recognition \cite{krizhevsky2012imagenet, simonyan2014very}. However, deep learning based methods heavily rely on in-domain labeled data for training, to be generalized to target domain data. Considering that collecting annotated data for every possible domain is labour-intensive and time-consuming, a feasible solution is unsupervised domain adaptation (UDA) \cite{ben2010theory, ganin2015, long2018conditional}, which transfers the knowledge from labeled source domain to unlabeled target domain by alleviating distribution discrepancy between them. The most common setting in UDA is closed-set DA which assumes the source class set $C_{s}$ is identical to the target class set $C_{t}$. 
	This may be impractical in real-world applications, because it is difficult to ensure that the target dataset always has the same classes as the source dataset.

	To tackle this problem, some works consider more general domain adaptation tasks. For example, partial domain adaptation (PDA) \cite{cao2018part} assumes that target class set is a subset of source class set, \ie, $C_{t} \subseteq C_{s}$. Open-set domain adaptation (OSDA) \cite{saito2018open} exploits the situation where source class set is a subset of target class set $C_{s}\subseteq C_{t}$. Universal domain adaptation (UniDA) \cite{you2019universal} is a more general setting that both source and target domains possibly have common and private classes. The setting of UniDA includes PDA, OSDA, and a mixture of PDA and OSDA, \ie, open-partial DA (OPDA) \cite{panareda2017open}, in which both source and target domains have private classes. 
	This paper focuses on the general UniDA setting. The goal of UniDA is to classify target domain common class samples and detect target-private class samples, meanwhile reducing the negative transfer possibly caused by source private classes. To reduce the domain gap in UniDA, we may use distribution alignment techniques as in UDA methods~\cite{courty2017joint,ganin2015}, to align the distributions of two domains.  However, matching all data of two domains may lead to the mismatch of the common class data of one domain to the private class data in the other domain, and cause negative transfer.
	
	In this work, we propose a novel Prototypical Partial Optimal Transport (PPOT) approach to tackle UniDA. Specifically, we 
	model distribution alignment in UniDA as a partial optimal transport (POT) problem, to align a fraction of data (mainly from common classes), between two domains using POT. We design a prototype-based POT, in which the source data are represented as prototypes in POT formulation, which is further formulated as a mini-batch-based version, dubbed m-PPOT. We prove that POT can be bounded by m-PPOT and the distances between source samples and their corresponding prototypes, inspiring us to design a deep learning model for UniDA by using m-PPOT as one training loss. Meanwhile, the transport plan of m-PPOT can be regarded as a matching matrix, enabling us to utilize the row sum and column sum of the transport plan to reweight the source prototypes and target samples for distinguishing “known” and “unknown” samples. Based on the transport plan of m-PPOT, we further design reweighted cross-entropy loss on source labeled data and reweighted entropy loss on target data to learn a transferable recognition model.  

    In experiments, we evaluate our method on four UniDA benchmarks. Experimental results show that our method performs favorably compared with the state-of-the-art methods for UniDA. The ablation study validates the effectiveness of individual components proposed in our method.
	
\section{Related Work}
	\subsection{Domain Adaptation}
	Unsupervised domain adaptation aims to reduce the gap between source and target domains. Previous works \cite{ganin2015, long2015learning, long2018conditional} mainly focus on distribution alignment to mitigate domain gaps. The theoretical analysis in \cite{ben2010theory} shows that minimizing the discrepancy between source and target distributions may reduce the target prediction error. Previous works often minimize distribution discrepancy between two domains by adversarial learning \cite{ganin2015, long2018conditional, zhang2019bridging} and moment matching \cite{long2015learning, sun2016deep, pan2019transferrable}. Partial DA \cite{cao2018partial} tackles the scenario that only the source domain contains private classes. The methods in \cite{cao2018part, zhang2018importance, gu2021adversarial} are mainly based on reweighting source data for reducing negative transfer caused by source private class samples. Open-set DA \cite{saito2018open} assumes that the label set of the source domain is a subset of that of the target domain. ~\cite{saito2018open, liu2019separate, bucci2020effectiveness} propose diverse methods to classify “known” samples meanwhile rejecting “unknown” samples. 
	
    \subsection{Universal Domain Adaptation}
	Universal DA does not have any prior knowledge on the label space of two domains, which means that both source and target domains may or may not have private classes. UAN \cite{you2019universal} computes the transferability of samples by entropy and domain similarity to separate “known” and “unknown” samples. CMU \cite{fu2020learning} improves UAN to measure transferability by a mixture of entropy, confidence, and consistency from ensemble model. DANCE \cite{saito2020universal} designs an entropy-based method by increasing the confidence of common class samples while decreasing it for private class samples to better distinguish known and unknown samples. DCC \cite{li2021domain} tries to exploit the domain consensus knowledge to discover matched clusters for separating common classes in cluster-level. OVANet \cite{saito2021ovanet} trains a “one-vs-all” discriminator for each class to recognize private class samples. GATE \cite{chen2022geometric} explores the intrinsic geometrical relationship between the two domains and designs a universal incremental classifier to separate “unknown” samples.
	Different from the above methods, we model the UniDA as a partial distribution alignment problem and propose a novel m-PPOT model to solve it. 
	
	\subsection{Optimal Transport} 
	Optimal transport (OT) \cite{villani2009optimal, peyre2019computational} is a mathematical tool for transporting/matching distributions. OT has been applied to diverse tasks such as generative adversarial training \cite{arjovsky2017wasserstein}, clustering \cite{ho2017multilevel}, domain adaptation \cite{courty2017joint}, object detection \cite{ge2021ota}, etc. The partial OT \cite{caffarelli2010free, figalli2010optimal} is a special OT problem that only transports a portion of the mass. To reduce computational cost of OT, the Sinkhorn OT \cite{cuturi2013sinkhorn} can be efficiently solved by the Sinkhorn algorithm, and is further extended to partial OT in \cite{benamou2015iterative}.  
	In \cite{flamary2016optimal, courty2017joint, damodaran2018deepjdot}, OT was applied to domain adaptation to align distributions of source and target domains in input space or feature space. They use OT in mini-batch to reduce computational overhead, however, suffering from sampling bias that the mini-batch data partially reflect the original data distribution.~\cite{fatras2021unbalanced, nguyen2022improving} replace mini-batch OT with more robust OT models, such as unbalanced mini-batch OT and partial mini-batch OT, and achieve better performance.~\cite{xu2021joint} designs joint partial optimal transport which only transports a fraction of the mass for avoiding negative transfer, and extends the task into open-set DA. 
	
	In this work, we consider the UniDA task. We propose a novel mini-batch based prototypical POT model, which partially aligns the source prototypes and target features to solve the problem of UniDA. Experiments show that our method achieves state-of-the-art results for UniDA.
	
	\begin{figure*}[t]
		\centering
		\includegraphics[width=0.935\textwidth]{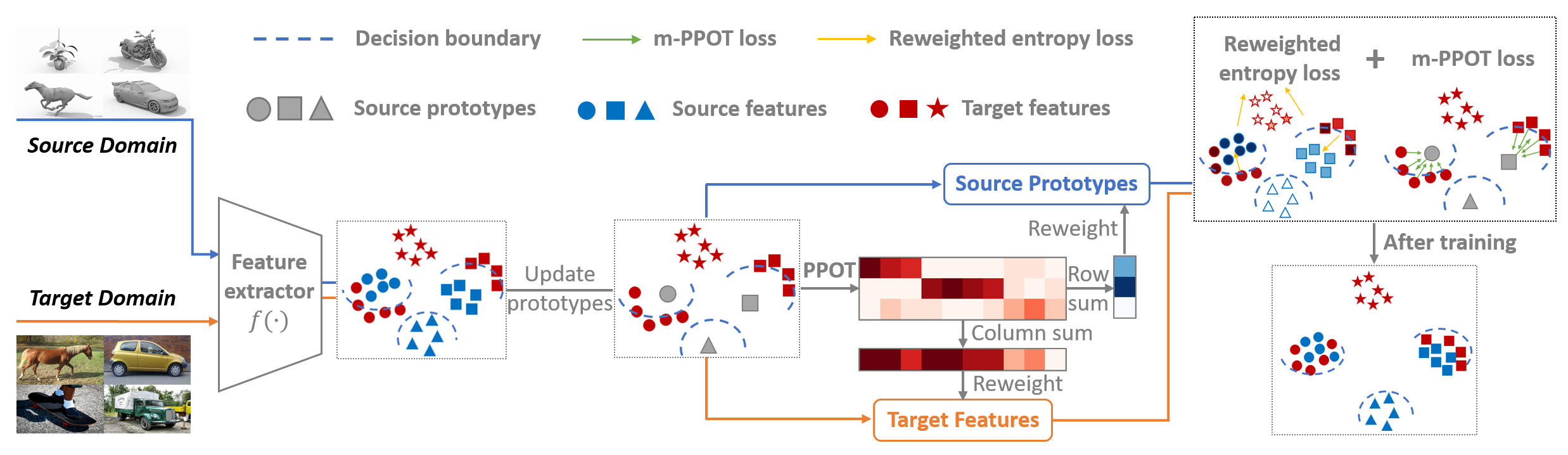}
		\caption{\textbf{Illustration of our model.} Source and target data share the same feature extractor that embeds data in feature space. PPOT is to match target features and source prototypes which are updated by the source features, and the row/column sum of transport plan is applied for reweighting. We design reweighted entropy loss to align common class features of two domains, while pushing away the unknown features.}
		\label{fig: model}
	\end{figure*}
	
	\section{Preliminaries on Optimal Transport}\label{sec: preliminaries}
	We consider two sets of data points, $\{x_{i}^{s} \}_{i=1}^{m}$ and $\{x_{j}^{t} \}_{j=1}^{n}$, of which the empirical distributions are denoted as $\bm{\mu}=\sum_{i=1}^{m}\mu_{i}\delta_{x_{i}^{s}}$ and $\bm{\nu}=\sum_{j=1}^{n}\nu_{j}\delta_{x_{j}^{t}}$ respectively, where $\sum_{i=1}^{m}\mu_{i} =1$, $\sum_{j=1}^{n}\nu_{j} = 1$ and $\delta_{x}$ is the Dirac function at position $x$. With a slight abuse of notations, we denote  $\bm{\mu}=(\mu_1,\mu_2,\cdots,\mu_m)^{\top}$, $\bm{\nu}=(\nu_1,\nu_2,\cdots,\nu_n)^{\top}$ and define a cost matrix as $C \in \mathbb{R}^{m\times n}, C_{ij} = c(x_{i}^{s}, x_{j}^{t})$.
	
	\textbf{Kantorovich problem.}
   The Kantorovich problem \cite{kantorovitch1958} aims to derive a transport plan from $\bm{\mu}$ to $\bm{\nu}$, modeled as the following linear programming problem:
	\begin{equation}\label{eq:kant}
		\begin{aligned}
			&\text{OT}(\bm{\mu}, \bm{\nu})\triangleq\min_{\pi \in \Pi(\bm{\mu},\bm{\nu})}\langle \pi , C \rangle_F \\ 
			&s.t. \hspace{0.2cm} \Pi(\bm{\mu},\bm{\nu}) = \{\pi \in  \mathbb{R}^{m \times n}_{+}\vert \pi \mathbb{1}_{n} = \bm{\mu}, \pi^{\top} \mathbb{1}_{m} = \bm{\nu}\},
		\end{aligned}
	\end{equation}
	where $\langle \cdot , \cdot \rangle_F$ denotes the Frobenius inner product.
	
	\textbf{Mini-batch OT} is designed to reduce computational cost and make OT more suitable for deep learning. We denote the collection of empirical distributions of $b$ random samples in $\{x_{i}^{s} \}_{i=1}^{m}$ (resp. $\{x_{j}^{t} \}_{j=1}^{n}$) as $\mathcal{P}_{b}(\bm{\mu})$ (resp. $\mathcal{P}_{b}(\bm{\nu})$), where $b$ is the batch size, and $k$ is the number of mini-batches.
	The mini-batch OT is defined as
	\begin{equation}\label{eq:mot}
	\text{m-OT}_{k}(\bm{\mu}, \bm{\nu}) = \frac{1}{k}\sum_{i=1}^{k}\text{OT}(A_{i}, B_{i}),
    \end{equation}
    where $A_{i} \in \mathcal{P}_{b}(\bm{\mu})$, $B_{i} \in \mathcal{P}_{b}(\bm{\nu})$ for any $i = 1, 2, ..., k$.
	
	\textbf{Partial OT} aims to transport only $\alpha$ mass $(0 \leqslant \alpha \leqslant \min(\Vert\bm{\mu}\Vert_{1}, \Vert\bm{\nu}\Vert_{1}))$ between $\bm{\mu}$ and $\bm{\nu}$ with the lowest cost. The partial OT is defined as
	\begin{equation}\label{eq:pot}
		\begin{aligned}
			&\text{POT}^{\alpha}(\bm{\mu}, \bm{\nu})\triangleq\min_{\pi \in \Pi^{\alpha}(\bm{\mu},\bm{\nu})}\langle \pi , C \rangle_F, \\
		\end{aligned}
	\end{equation}
	where $	\Pi^{\alpha}(\bm{\mu},\bm{\nu}) = \{\pi \in \mathbb{R}^{m\times n}_+ \vert \pi \mathbb{1}_{n} \leqslant \bm{\mu}, \pi^{\top} \mathbb{1}_{m} \leqslant \bm{\nu}, \mathbb{1}_{m}^{\top}\pi\mathbb{1}_{n} = \alpha\}$.

	\section{Method}
	In this section, we model UniDA as a partial distribution alignment problem. To partially align source and target distributions, the mini-batch prototypical partial optimal transport (m-PPOT) is proposed. The m-PPOT focuses on the discrete partial OT problem between source prototypes and target samples for mini-batch. Based on m-PPOT, we design a novel model for UniDA.
    We also use contrastive pre-training to have a better initialization of network parameters. 
	
	In UniDA, we are given labeled source data $\bm{D_{s}} =  \{x_{i}^{s}, y_{i}\}_{i=1}^{m}$ and unlabeled target data $\bm{D_{t}}=\{x_{j}^{t}\}_{j=1}^{n}$ . UniDA aims to label the target sample with a label from source class set $C_{s}$ or discriminate it as an “unknown” sample. We denote the number of source domain classes as $L = |C_{s}|$. Our deep recognition model consists of two modules, including a feature extractor $f$ mapping input $x$ into feature $z$, and an $L$-way classification head $h$. The source and target empirical distributions in feature space are denoted as $\bm{\bar{p}} = \sum_{i=1}^{m}p_{i}\delta_{f(x_{i}^{s})}$ and $\bm{\bar{q}} = \sum_{j=1}^{n}q_{j}\delta_{f(x_{j}^{t})}$ respectively, where $\sum_{i=1}^{m}p_{i} =1$, $\sum_{j=1}^{n}q_{j} = 1$. With a slight abuse of notations, we denote the vector of data mass as $\bm{\bar{p}} = (p_{1}, p_{2}, ... , p_{m})^{\top}$ and $\bm{\bar{q}} = (q_{1}, q_{2}, ... , q_{n})^{\top}$ and we set $p_i = \frac{1}{m}, ~q_j = \frac{1}{n}$ for any $i, j$ in this paper. Furthermore, the element of cost matrix $C$ is defined as $C_{ij} = d(f(x_{i}^{s}), f(x_{j}^{t}))$, where $d$ is the $L_2$-distance.
	
	\subsection{Modeling UniDA as Partial OT}
	\cite{ben2006analysis, ben2010theory} presented theoretical analysis on domain adaptation, emphasizing the importance of minimizing distribution discrepancy. 
 However, it can not be simply extended to UniDA because the source/target data may belong to source/target private classes in UniDA setting. Directly aligning source distribution $\bm{\bar{p}}$ and target distribution $\bm{\bar{q}}$ will lead to data mismatch due to the existence of “unknown” samples in both domains. For UniDA task, we first decompose $\bm{\bar{p}}$, $\bm{\bar{q}}$ as
	\begin{equation*}
		\bm{\bar{p}} = (1 - \beta) \bm{p}_{p} + \beta  \bm{p}_{c}, \quad
		\bm{\bar{q}} = (1 - \alpha) \bm{q}_{p} + \alpha  \bm{q}_{c},
	\end{equation*}
	where $\bm{p}_{p}$ (resp. $\bm{q}_{p}$) denotes distribution of source (resp. target) private class data in feature space, $\bm{p}_{c}$ and $\bm{q}_{c}$ are denoted as source and target common class data distributions, $\alpha$ and $\beta$ are the ratio of common class samples in the source and target domain respectively. Our goal is to minimize the discrepancy between $\bm{p}_{c}$ and $\bm{q}_{c}$, formulated as an OT problem:
	\begin{equation}\label{vanilla}
		\min_{f, \pi} \langle \pi, \bar{C} \rangle_{F} = \min_{f} \text{OT}(\bm{p}_{c}, \bm{q}_{c}),
    \end{equation}
    where $\bar{C} \in \mathbb{R}^{|\bm{p}_{c}| \times |\bm{q}_{c}|}$ is a submatrix of $C$, corresponding to the common class samples.
	
	Obviously, we can not directly get these two distributions.
    Therefore, we consider to find an approximation of Eqn.~(\ref{vanilla}). Following the assumption in \cite{you2019universal} that $\bm{q}_{c}$ is closer to $\bm{p}_{c}$ than $\bm{q}_{p}$, meaning that the cost of transport between two domains' common class samples is generally less than the cost between two private class samples of these two domains or the private and common class samples of them. Note that partial OT only transports a fraction of the mass  having lowest cost to transport. With the above assumption, the partial transport between two domains will prefer to transfer the common class samples of them. Therefore, we approximately solve Eqn. (\ref{vanilla}) by optimizing 
	\begin{equation}\label{v2p}
		\min_{f} \text{POT}^{\alpha}(\frac{\alpha}{\beta}\bm{\bar{p}}, \bm{\bar{q}})
	\end{equation}
where coefficient $(\alpha/\beta)$ is to ensure that the mass of common class samples in $\bm{\bar{p}}$ and $\bm{\bar{q}}$ equals. The superscript $\alpha$ denotes the total mass to transport. For convenience of presentation, $(\alpha/\beta)\cdot\bm{\bar{p}}$ and $\bm{\bar{q}}$ are denoted as $\bm{p}$ and $\bm{q}$ respectively.
	
	\subsection{Prototypical Partial Optimal Transport}
	We have turned the distribution alignment between $\bm{p}_{c}$ and $\bm{q}_{c}$ into a partial OT problem in Eqn.~(\ref{v2p}). The remaining challenge is to embed partial OT into a deep learning framework. In this paper, we design a mini-batch based prototypical partial optimal transport problem for UniDA. We first define the Prototypical Partial Optimal Transport (PPOT).
	
	\begin{define}
		(Prototypical Partial Optimal Transport) Let  $\{c_{i}\}_{i=1}^{L}$ be the set of source domain prototypes, defined as 
		\begin{equation*}
			c_{i} = \sum_{j: y_{j}=i}\dfrac{f(x_{j}^{s})}{\sum_{l=1}^{m}\bm{1}(y_{l}=i)}.
		\end{equation*}
		The element of cost matrix $C_{ij}$ is defined as $d(c_{i}, f(x_{j}^{t}))$ and the PPOT transportation cost between $\bm{p}$ and $\bm{q}$ is defined as
		\begin{equation}
			\textup{PPOT}^{\alpha}(\bm{p}, \bm{q}) \triangleq \textup{POT}^{\alpha}(\bm{c}, \bm{q}) = \min_{\pi \in \Pi^{\alpha}(\bm{c},\bm{q})}\langle \pi , C \rangle_F,
		\end{equation}
		where $\bm{c} = \sum_{i=1}^{L}r_{i}\delta_{c_{i}}$ is the empirical distribution of source domain prototypes, and $r_{i} = \sum_{j: y_{j}=i}p_{j}$.
	\end{define}
		PPOT is suitable for the DA task because, first, it fits the mini-batch based deep learning implementation in which all of the prototypes, instead of batch of source samples, are regarded as source measures in POT. This change could reduce the mismatch caused by the lack of full coverage of source samples in a batch, and second, it requires less computational resources than original POT. 
  
  \subsubsection{Mini-batch based PPOT.} We further extend the PPOT to the mini-batch version m-PPOT, here we assume batch size $b$ satisfy $b\mid n$ and set $k = n / b$. Let $\mathcal{B}_{i}$ be the $i$-th index set of $b$ random target samples and their corresponding empirical distribution in feature space is denoted as $\bm{q}_{\mathcal{B}_{i}}$. We define $\mathcal{B} \triangleq \{\mathcal{B}_{i}\}_{i=1}^{k}$ as a partition if they satisfy:
	\begin{itemize}
		\item $\mathcal{B}_{i} \bigcap \mathcal{B}_{j} = \emptyset : \forall \; 0 \leqslant i < j \leqslant k$
		\item $\bigcup\limits_{i=1}^{k}\mathcal{B}_{i} = \{1,2,...,n\}$	
	\end{itemize}
	and the m-PPOT is defined as
	\begin{equation}\label{eq:mppot}
		\text{m-PPOT}^{\alpha}_{\mathcal{B}}(\bm{p}, \bm{q}) \triangleq \frac{1}{k}\sum_{i=1}^{k}\text{POT}^{\alpha}(\bm{c}, \bm{q}_{\mathcal{B}_{i}}),~~ \mathcal{B} \in \Gamma
	\end{equation}
	where $\Gamma$ is the set of all partitions of $\{1,2,...,n\}$, \ie, the index set of target data. Note that these assumptions are easily satisfied by the dataloader module in pytorch. Furthermore, we denote the optimal transportation in $i$-th batch as $\pi^{\alpha}_{i}$. To show that m-PPOT is closely related to PPOT, we give the following proposition \ref{prop: transportation}.
	\begin{proposition}\label{prop: transportation}
		We extend $\pi_{i}^{\alpha}$ to a $L \times n$ matrix $\Pi_{i}^{\alpha}$  that pad zero entries to the column whose index does not belong to $\mathcal{B}_{i}$, then we have
		\begin{equation*}			
			\dfrac{1}{k}\sum_{i=1}^{k}\Pi^{\alpha}_{i} \in \Pi^{\alpha}(\bm{c},\bm{q})			
		\end{equation*}
		and
		\begin{equation}			
			\textup{PPOT}^{\alpha}(\bm{p}, \bm{q}) \leqslant \textup{m-PPOT}^{\alpha}_{\mathcal{B}}(\bm{p}. \bm{q}).				
		\end{equation}
	\end{proposition}
	 Proposition \ref{prop: transportation} implies that $\textup{m-PPOT}^{\alpha}_{\mathcal{B}}(\bm{p}, \bm{q})$ is an upper bound for $\text{PPOT}^{\alpha}(\bm{p}, \bm{q})$.
	 The following theorem shows that POT is bounded by the sum of m-PPOT and the distances of source samples to their corresponding prototypes.

	\begin{thm}\label{thm: thm2}
		Considering two distributions $\bm{p}$ and $\bm{q}$, the distance between $f(x_{i}^{s})$ and corresponding prototype $c_{y_{i}}$ is denoted as $d_{i} \triangleq d(f(x_{i}^{s}), c_{y_{i}})$. The row sum of the optimal transport plan of $\text{PPOT}^{\alpha}(\bm{p}, \bm{q})$ is denoted as $\bm{w} = (w_1, w_2, ... , w_L )^{\top}$, $r_{i} = \sum_{j: y_{j}=i}p_{j}$. Then we have 
		\begin{equation}
			\textup{POT}^{\alpha}(\bm{p}, \bm{q}) \leqslant \sum_{i=1}^{m} \dfrac{w_{y_i}}{r_{y_i}} p_i d_{i} + \textup{m-PPOT}^{\alpha}_{\mathcal{B}}(\bm{p}, \bm{q}).
		\end{equation}
	\end{thm}
	The proofs of theorem \ref{thm: thm2} and proposition \ref{prop: transportation} are included in Appendix.
	\subsection{UniDA based on m-PPOT}
	Our motivation is to minimize discrepancy between distributions of source and target common class data, meanwhile separating “known” and “unknown” data in both domains in training. We design the following losses for training. 
 
	\subsubsection{m-PPOT loss.}
	Based on theorem \ref{thm: thm2}, to minimize the discrepancy between $\bm{p}_{c}$ and $\bm{q}_{c}$, we first design the m-PPOT loss to minimize the second term in the bound of theorem \ref{thm: thm2}. We introduce the $\text{m-PPOT}^{\alpha}_{\mathcal{B}}(\bm{p}, \bm{q})$ as a loss:
	\begin{equation}\label{eq: otloss}
		 \mathcal{L}_{ot}  = \mathop{\mathbb{E}} \limits_{\mathcal{B} \in \Gamma}\left(\text{m-PPOT}^{\alpha}_{\mathcal{B}}(\bm{p}, \bm{q})\right),
	\end{equation}
	where $\mathop{\mathbb{E}}$ denotes the expectation over all target domain data index partitions in $\Gamma$. Using the mini-batch based optimization method, this term can be approximated by the partial OT problem $\text{POT}^\alpha(\bm{c}, \bm{q}_{\mathcal{B}_i})$ over each mini-batch, according to Eqn.~(\ref{eq:mppot}). The set of prototypes $\bm{c}$ is updated by exponential moving average as in \cite{xie2018learning}. We use the entropy regularized POT algorithm proposed by \cite{benamou2015iterative} to solve POT on mini-batch.
 
	\subsubsection{Reweighted entropy loss.}
 We further design entropy-based loss on target domain data to increase the prediction certainty. The solution $\pi^{*}$ to the $\text{m-PPOT}^{\alpha}_{\mathcal{B}}(\bm{p}, \bm{q})$ is a matrix measuring the matching between source prototypes and target features. Since the more easily a prototype (feature) can be transported, the more likely it belongs to a common class (“known” sample), we leverage the row/column sum of $\pi^{*}$ as indicator to identify unknown samples. Specifically, we first get the column sum of $\pi^{*}$ and multiply a constant $n / \alpha$ to make $\bm{w}^{t} \in \mathbb{R}^{n}$ satisfy $\Vert \bm{w}^{t}\Vert_{1} = n$. A reweighted entropy loss is formulated as 
	\begin{equation}\label{eq: peloss}
		\begin{gathered}
			\mathcal{L}_{pe}  = - \sum_{i=1}^{n}\sum_{j=1}^{L} w^{t}_{i} p_{ij} \log (p_{ij}), 
			w^{t}_{i} = \dfrac{n}{\alpha}\sum_{j=1}^{L}\pi^{*}_{ij},
		\end{gathered}		
	\end{equation}
	where $p_{ij}  \triangleq \sigma(h\circ f(x_{i}^{t}))_{j}$. We take this loss to increase the confidence of prediction for those target samples  seen as “known” samples. 
	
	Furthermore, we follow \cite {saito2020universal, saito2021ovanet} to suppress the model to generate overconfident predictions for target “unknown” samples by loss
	\begin{equation}\label{eq: neloss}
		\begin{gathered}
			\mathcal{L}_{ne} = - \sum_{i=1}^{n}\sum_{j=1}^{L} w^{u}_{i} p_{ij} \log (p_{ij}), 
			w^{u}_{i} = [1 - w^{t}_{i}]_{+},
		\end{gathered}
	\end{equation}
	where $w^{u}_{i}$ depends on $w^{t}_{i}$, and higher $w^{u}_{i}$ for a sample means higher confidence to be an “unknown” sample. Therefore, we use $\mathcal{L}_{ne}$ to reduce the confidence of those samples which are likely to be “unknown” samples.
	\subsubsection{Reweighted cross-entropy loss.}
	This loss is the classification loss defined in the source domain, based on the cross-entropy using labels of source domain data. Different to standard classification loss, we use the column sum of $\pi^{*}$ to compute weights $\bm{w}^{s}\in \mathbb{R}^{L}$ for measuring the confidence of the ``known'' source domain prototypes. Then we design the reweighted cross-entropy loss 
	\begin{equation}\label{eq: rceloss}
		\begin{gathered}
			\mathcal{L}_{rce} = -\sum_{i=1}^{m}\sum_{j=1}^{L}w^{s}_{j}\bm{1}(y_{i}=j)\log(\sigma(h\circ f(x_{i}^{s}))_{j})
		\end{gathered}
	\end{equation}
where $w^{s}_{j} = \dfrac{L}{\alpha}\sum_{i=1}^{n}\pi^{*}_{ij}$ is the weight of $j$-th source prototype representing $j$-th class center. The weights satisfy $\sum_{j=1}^{L} w^{s}_{j} = L$ and each of them represents the possibility that each category belongs to a common class. 
\cite{papyan2020prevalence} shows that the cross-entropy based loss could minimize the distance of features to class prototype. 
This implies that the reweighted cross-entropy loss approximately minimizes the first term in bound of theorem \ref{thm: thm2}, in which we use the row sum $\bm{w}^s$ of m-PPOT to approximate the row sum $\bm{w}$ of PPOT, and use the class-balanced sampling in implementation to enforce that $r_j, \forall j,$ are equal.
	
	\subsubsection{Training loss and details.}
	Our model is jointly optimized with the above loss terms, and the total training loss is
	\begin{equation} \label{eq: total}
		\mathcal{L} = \mathcal{L}_{rce} + \mathcal{L}_{ent} + \eta_1 \mathcal{L}_{ot},
	\end{equation}
	where $\mathcal{L}_{ent} = \eta_2 \mathcal{L}_{pe} - \eta_3 \mathcal{L}_{ne}.
	$ In implementation, we set
 $\eta_1 = 5$, $\eta_2 = 0.01$, and $\eta_3 = 2$ for all datasets.
The training process of our method is shown in Fig. \ref{fig: model}. In the beginning, we map data in both domains into feature space by the feature extractor. Source prototypes are updated by source features in every batch and then we compute the m-PPOT between the empirical distributions of source prototypes and target samples, which we also leverage the row sum and column sum of corresponding transport plan to reweight in the losses. $\mathcal{L}_{ot}$ aims to reduce the gap between the distribution of “known” samples in both domains, meanwhile $\mathcal{L}_{ent}$ enforces the ``known'' samples to have higher prediction confidence by decreasing their entropy, and the ``unknown'' samples to have lower prediction confidence by increasing the entropy, in the target domain. Since the classifiers are learned over the source domain data, this may align the ``known'' target domain data to the source domain data distribution, while pushing the ``unknown'' target domain data away from the source domain data distribution.

    \begin{table*}[!htbp]
		\centering
		\setlength{\tabcolsep}{8.5pt}
		\begin{tabular}{l|c|c|c|cccccc|c}
			\toprule
			\multirow{2}*{Method} & Office-31 &Office-Home&VisDA & \multicolumn{6}{c}{DomainNet (150/50/145)} & ~\cellcolor{white}\\
			\cmidrule{5-11}
			~&(10/10/11)\cellcolor{white}&(10/5/50)\cellcolor{white}&(6/3/3)\cellcolor{white}&P$\rightarrow$R&P$\rightarrow$S&R$\rightarrow$P&R$\rightarrow$S&S$\rightarrow$P&S$\rightarrow$R&Avg\\
			\midrule 
			UAN & 63.5&56.6&30.5& 41.9 & 39.1 & 43.6 & 38.7 & 39.0 & 43.7 & 41.0\\
			CMU &73.1&61.6& 34.6& 50.8 & 45.1 & 52.2 & 45.6 & 44.8 & 51.0 & 48.3\\
			DANCE & 82.3&63.9&42.8& 55.7 & 47.0 & 51.1 & 46.4 & 47.9 & 55.7 & 50.6\\
			DCC &  80.2&70.2&43.0&56.9 & 43.7 & 50.3 & 43.3 & 44.9 & 56.2 & 49.2\\
			OVANet & 86.5&71.8&53.1& 56.0 & 47.1 & 51.7 & 44.9 & 47.4 & 57.2 & 50.7\\
			GATE & 87.6&75.6&56.4& 57.4 & 48.7 & 52.8 & 47.6 & 49.5 & 56.3 & 52.1\\
			\midrule
			\bf PPOT& \bf90.4&\bf77.1&\bf73.8& \bf 67.8 & \bf 50.2 & \bf 60.1 & \bf 48.9 & \bf 52.8 & \bf 65.4 & \bf 57.5\\
			\bottomrule
		\end{tabular}
		\caption{H-score (\%) comparison on Office-31, Office-Home, VisDA and DomainNet for OPDA. Note that we only report the average H-score over all tasks on Office-31 on Office-Home, and the results for different tasks are in Appendix.
		}	
		\label{tab:OPDA}	
	\end{table*}

 \subsubsection{Parameter initialization by contrastive pre-training.}
Motivated by \cite{shen2022connect}, we use contrastive learning to pre-train our feature extractor. Specifically, we send both source and target unlabeled data into our feature extractor and use the contrastive learning method (MocoV2 \cite{chen2020improved}) to pre-train our feature extractor, then fine-tune the entire model on labeled source data, and take these parameters as our model's initial parameters. We empirically find that contrastive pre-training also works in UniDA setting in our experiments.

	\subsection{Hyper-parameters}
	We notice that $\alpha$ and $\beta$ are nearly impossible to calculate precisely in practice, so we propose a method to compute them approximately. We denote two scalars as $\tau_1$ and $\tau_2$, where $\tau_1 \in (0, 1]$, $\tau_2 > 0$. To simplify the notation, we use $s(x) = \max \sigma(h\circ f(x))$ to denote the prediction confidence of $x$. We define $\alpha$ and $\beta$ as  
	\begin{equation} \label{eq: alpha}
		\alpha = \sum_{j=1}^{n}   \dfrac{\bm{1}(s(x_{j}^{t}) \geqslant \tau_1)}{n}, 
		\beta = \sum_{i=1}^{L}   \dfrac{\bm{1}(w^{s}_{i} \geqslant \tau_2)}{L}.
	\end{equation}
	The motivation is that we use the proportion of high-confidence samples to estimate the ratio of “known” samples in the target domain, and similarly use the proportion of categories with high weights to approximate the ratio of common classes in the source domain.

	In experiments, we set $\tau_1 = 0.9$ and $\tau_2 = 1$. In the $i$-th iteration of training phase, we first calculate $\alpha^{i}$ by Eqn. (\ref{eq: alpha}), and update $\alpha$ by exponential moving average:
	\begin{equation*}
		\alpha^i \leftarrow \lambda_1 \alpha^i + (1 - \lambda_1) \alpha^{i-1}.
	\end{equation*}
	Then we use $\alpha^i$ as transport ratio and $\alpha^i / \beta^{i-1}$ as coefficient of Eqn. (\ref{v2p}) to compute $\mathcal{L}_{ot}$ and its by-product $\bm{w}^s$. After that we compute $\beta^i$ by Eqn. (\ref{eq: alpha}) and update it as same as $\alpha^i$:
	\begin{equation*}
		\beta^i \leftarrow \lambda_2 \beta^i + (1 - \lambda_2) \beta^{i-1},
	\end{equation*}
	where $\lambda_1, \lambda_2 \in [0, 1)$ are set to 0.001 in our experiments.
	
	Furthermore, to reduce the possible mistakes that identify a “known” sample as “unknown” sample, we retain only a fraction of $\{w^{u}_{i}\}_{i=1}^n$ that have larger values, and set the others as 0. The fraction is set to 25\% in all tasks.
 
	\section{Experiment}
	We evaluate our method on UniDA benchmarks. We solve three settings of UniDA, including OPDA, OSDA, and PDA but without using prior knowledge about the mismatch of source and target domain class label sets.
	\subsubsection{Datasets.} 
	\textbf{Office-31} \cite{office31} includes 4652 images in 31 categories from 3 domains: Amazon (\textbf{A}), DSLR (\textbf{D}), and Webcam (\textbf{W}). \textbf{Office-Home} \cite{officehome} consists of 15500 images in 65 categories, and it contains 4 domains: Artistic images (\textbf{A}), Clip-Art images (\textbf{C}), Product images (\textbf{P}), and Real-World images (\textbf{R}). \textbf{VisDA} \cite{visda} is a larger dataset which consists of 12 classes, including 150,000 synthetic images (\textbf{S}) and 50,000 images from real world (\textbf{R}). \textbf{DomainNet} \cite{domainnet} is one of the most challenging datasets in DA task with about 0.6 million images, which consists of 6 domains sharing 345 categories. We follow \cite{fu2020learning} to use 3 domains: Painting (\textbf{P}), Real (\textbf{R}), and Sketch (\textbf{S}).  Following  \citep{saito2021ovanet}, we show the number of common classes, source private classes, and target private classes in brackets in the header of each result of tables.
	\subsubsection{Evaluation.} 
	In PDA tasks, we compute the accuracy for all target samples. In OSDA and OPDA settings, the target private class samples should be classified as a single category named “unknown”. The samples with confidence less than threshold $\xi$ are identified as “unknown”, where $\xi$ is set to 0.75 in all experiments.
	Following \cite{fu2020learning}, we report the H-score metric for OSDA and OPDA which is the harmonic mean of the average accuracy on common and private class samples. 

\begin{table}[t]
		\centering
		\setlength{\tabcolsep}{4.5pt}
		\begin{tabular}{l|c|c|c}
			\hline
			\multirow{1}*{Method} & Type & {Office-Home} (25/40/0) & VisDA (6/6/0)\\
			\hline
			PADA & P & 62.1 & 53.5\\
			IWAN & P & 63.6 & 48.6\\
			ETN & P & 70.5 & 59.8\\
			AR & P & \bf 79.4 & \bf 88.8\\
			\hline
			\hline
			DCC & U & 70.9 & 72.4\\
			GATE & U & 73.9 & 75.6\\
			\bf PPOT& U & \bf 74.3 & \bf 83.0\\		
			\hline
		\end{tabular}
		\caption{Comparison of H-score (\%)  on Office-Home and VisDA for PDA setting. ``P'' and ``U'' denote PDA and 
UniDA methods, respectively with and without assuming the target label set is a subset of the source label set. 
		}
		\label{tab:PDA_OH_VS}
	\end{table}
 
	\subsubsection{Implementation.}
	We implement our method using Pytorch \cite{paszke2019pytorch} on a single Nvidia RTX A6000 GPU. Following previous works \cite{saito2021ovanet, chen2022geometric}, we use ResNet50 \cite{he2016deep} without last fully-connected layer as our feature extractor. A 256-dimensional bottleneck layer and prediction head $h$ is successively added after the feature extractor. We use MocoV2 \cite{chen2020improved} to contrastive pre-train our feature extractor, the number of epochs in pre-training is 100, batch size is 256, and learning rate is 0.03. 
	
	In training phase, we optimize the model using Nesterov momentum SGD with momentum of 0.9 and weight decay of $5 \times 10^{-4}$. Following \cite{ganin2015}, the learning rate decays with the factor of $(1+\alpha t)^{-\beta}$, where $t$ linearly changes from 0 to 1 in training, and we set $\alpha = 10, \beta = 0.75$. The batch size is set to 72 in all experiments except in DomainNet tasks where it is changed to 256. We train our model for 5 epochs (1000 iterations per epoch), and update source prototypes and $\alpha$ totally before every epoch. The initial learning rate is set to $1 \times 10^{-4}$ on Office-31, $5 \times 10^{-4}$ on Office-Home and VisDA, and 0.01 on DomainNet.

\subsection{Results and Comparisons}

		We compare our method with four PDA methods (PADA \cite{cao2018partial}, IWAN \cite{zhang2018importance}, ETN \cite{cao2019learning}, AR \cite{gu2021adversarial}), three OSDA methods (OSBP \cite{saito2018open}, STA \cite{liu2019separate}, ROS \cite{bucci2020effectiveness}) and six UniDA methods (UAN \cite{you2019universal}, CMU \cite{fu2020learning}, DANCE \cite{saito2020universal}, DCC \cite{li2021domain}, OVANet \cite{saito2021ovanet}, GATE \cite{chen2022geometric}). All the compared methods use the same backbone as ours. 
	
	    \begin{table}[t]
		\centering
		\setlength{\tabcolsep}{4.5pt}
		\begin{tabular}{l|c|c|c}
			\hline
			\multirow{1}*{Method} & Type & {Office-Home} (25/0/40) & VisDA (6/0/6)\\
			\hline
			STA & O & 61.1 & 64.1\\
			OSBP & O & 64.7 & 52.3\\
			ROS & O & \bf 66.2 & \bf 66.5\\
			\hline
			\hline
			DCC & U &61.7 & 59.6\\
			OVANet & U & 64.0 & 66.1\\
			GATE & U & 69.1 & 70.8\\
			\bf PPOT& U & \bf 70.0 & \bf 72.3\\		
			\hline
		\end{tabular}
		\caption{Comparison of H-score (\%)  on Office-Home and VisDA for OSDA setting. ``O'' and ``U'' denote OSDA and UniDA methods, respectively with and without assuming the source label set is a subset of the target label set. 
		}
		\label{tab:OSDA_OH_VS}
	\end{table}
 
	\subsubsection{OPDA setting.}
	Table \ref{tab:OPDA} shows the results of our method. Our method outperforms baselines and achieves state-of-the-art results on all four datasets. On Office-31 and Office-Home datasets, our method surpasses all baselines on average. In larger datasets, VisDA and DomainNet, our method brings more than 17\% improvement over previous methods on VisDA, and 5\% on DomainNet. In general, these results show that our method is suitable in UniDA tasks, especially on larger and challenging datasets.  

	\subsubsection{PDA and OSDA settings.}
	Following \cite{li2021domain}, we train our model without any prior knowledge of label space mismatch in PDA and OSDA settings. We report the results for PDA setting in Table \ref{tab:PDA_OH_VS}. We can see that our method achieves better results than other UniDA-based methods (denoted as ``U'') on both datasets. The ``P'' denotes the PDA methods using prior knowledge that only the source domain has private classes. The results of OSDA setting are shown in Table \ref{tab:OSDA_OH_VS}, our method still surpasses all UniDA methods and OSDA methods (denoted as ``O'') using prior knowledge on label space mismatch on Office-Home and VisDA datasets.

 	\subsection{Model Analysis}
	\subsubsection{Comparison of m-PPOT with m-POT.}
	To compare m-PPOT with m-POT (mini-batch based partial OT without using prototypes) in UniDA, we replace m-PPOT with m-POT in our method and use the average weight of samples in each class to replace the prototype weights $\bm{w}^s$ in Eqn.~(\ref{eq: rceloss}), and the corresponding method is denoted as ``POT''. As shown in Table \ref{tab:ablation studies}, PPOT surpasses POT in all three datasets, confirming that m-PPOT performs better than m-POT in UniDA.

  	\begin{table}[t]
		\centering
		\setlength{\tabcolsep}{4pt}
		\begin{tabular}{l|c|c|c}
			\hline
			Method & Office-31 &  VisDA & Office-Home\\
			\hline
			POT & 88.4 & 66.4 & 74.2\\
			PPOT (w/o CL) & 89.4 & 58.1 &  74.3\\
			PPOT (w/o $\mathcal{L}_{pe}$) & 88.4 & 71.1 & 76.5 \\
			PPOT (w/o $\mathcal{L}_{ne}$) & 89.6 & 67.8 & 74.4 \\ 
			PPOT (w/o reweight) & 86.5 & 69.9 & 74.7 \\ 	
			\bf PPOT & \bf90.4 & \bf73.8 & \bf77.1\\

			\hline
		\end{tabular}
		\caption{Ablation study for OPDA on Office-31, Office-Home and VisDA. ``CL'' means contrastive pre-training.}
		\label{tab:ablation studies}		
	\end{table}

	\subsubsection{Effect of contrastive pre-training.}
	To evaluate the effect of contrastive pre-training, the contrastive pre-training is removed and the feature extractor is replaced by a ResNet50 pre-trained on ImageNet. The results shown in Table \ref{tab:ablation studies} illustrate that performance degenerates in all experiments, especially in more challenging tasks such as VisDA. 
	Note that without contrastive pre-training, our model still surpasses state-of-the-art methods on Office-31 and VisDA and reaches a comparable result on Office-Home.

 	\begin{figure}[t]
		\centering
		\subfigure[Source domain]{ \includegraphics[width=0.224\textwidth]{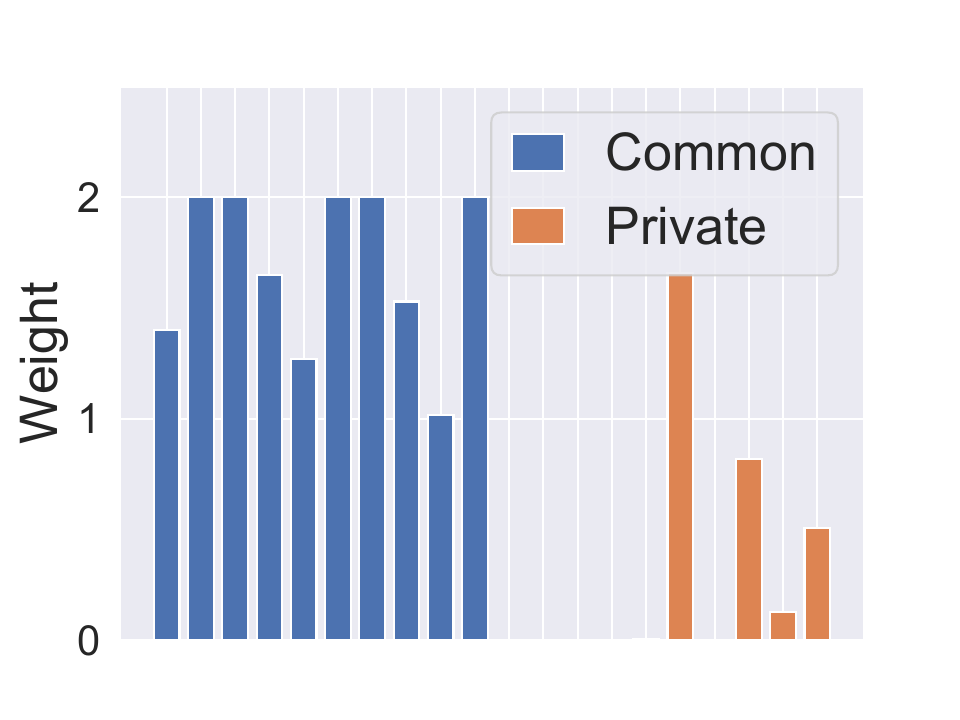} 
			\label{fig:source_weight}}
		\subfigure[Target domain]{ \includegraphics[width=0.224\textwidth]{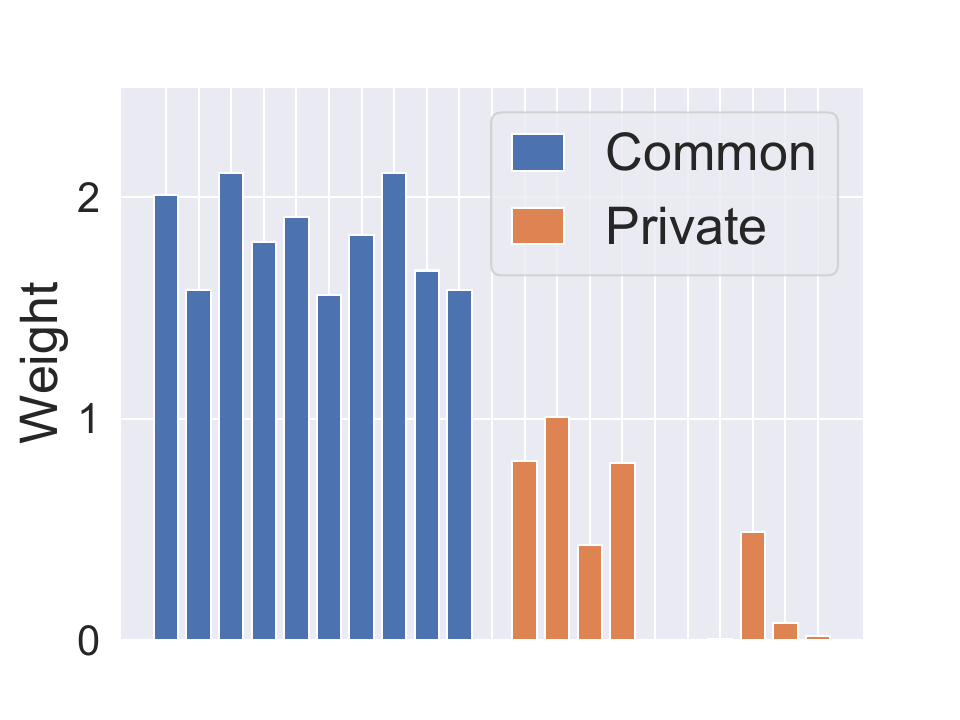} 
			\label{fig:target_weight}}
		\caption{(a) Class weight $\bm{w}^s$ in Eqn. (\ref{eq: rceloss}) on the source domain. (b) Average weight $\bm{w}^t$ in Eqn. (\ref{eq: peloss}) for each class on the target domain.  Task: W$\rightarrow$D on Office-31 for OPDA.}
		\label{fig:weight}
	\end{figure}

    \subsubsection{Effectiveness of reweighted entropy loss.}
	To evaluate this loss, we remove $\mathcal{L}_{pe}$ and $\mathcal{L}_{ne}$ in our model respectively. Table \ref{tab:ablation studies} shows that PPOT outperforms PPOT(w/o $\mathcal{L}_{pe}$) by 2\% and  PPOT(w/o $\mathcal{L}_{ne}$) by 6\% on VisDA dataset.

	\subsubsection{Effectiveness of reweighting strategy.}
	To evaluate the effectiveness of our reweighting strategy in Eqns.~(\ref{eq: rceloss}) and (\ref{eq: peloss}), we set $w_i^s = 1$ in Eqn.~(\ref{eq: rceloss}) and $w_j^t = 1$  in Eqn.~(\ref{eq: peloss}) for any $0\leqslant i \leqslant m$ and $0\leqslant j \leqslant n$. The results of Table \ref{tab:ablation studies} show that PPOT(w/o reweight) decreases at least 3\% more than PPOT in all datasets, which means that our reweighting strategy is important in our method. We further visualize the learned weights of source/target classes in UniDA task W$\rightarrow$D on Office-31 datasets, as shown in Fig.~\ref{fig:weight}. In both domains, most common classes have higher weights than private classes, which implies that our model can separate common and private classes effectively.

	\subsubsection{Sensitivity to hyper-parameters.}
	Figure \ref{fig:hyper-parameters} evaluates the sensitivity of our model to hyper-parameters $\tau_1$, $\tau_2$, $\eta_1$, $\eta_2$, $\eta_3$, and $\xi$. Results show that our model is relatively stable to $\tau_1$ and $\tau_2$ at the range of $[0.6, 0.95]$ and $[0.7, 1.1]$ respectively, as shown in Fig. \ref{fig:tau}. In Fig. \ref{fig:tau}, we can also see that the setting of threshold $\xi$ does not impact the performance much on our model in range $[0.5,0.9]$. Furthermore, Fig.~\ref{fig:eta} shows that our model is relatively stable to varying values of $\eta_1$, $\eta_2$, and $\eta_3$.

  \begin{figure}[t]
		\centering
		\subfigure[]{ \includegraphics[width=0.223\textwidth]{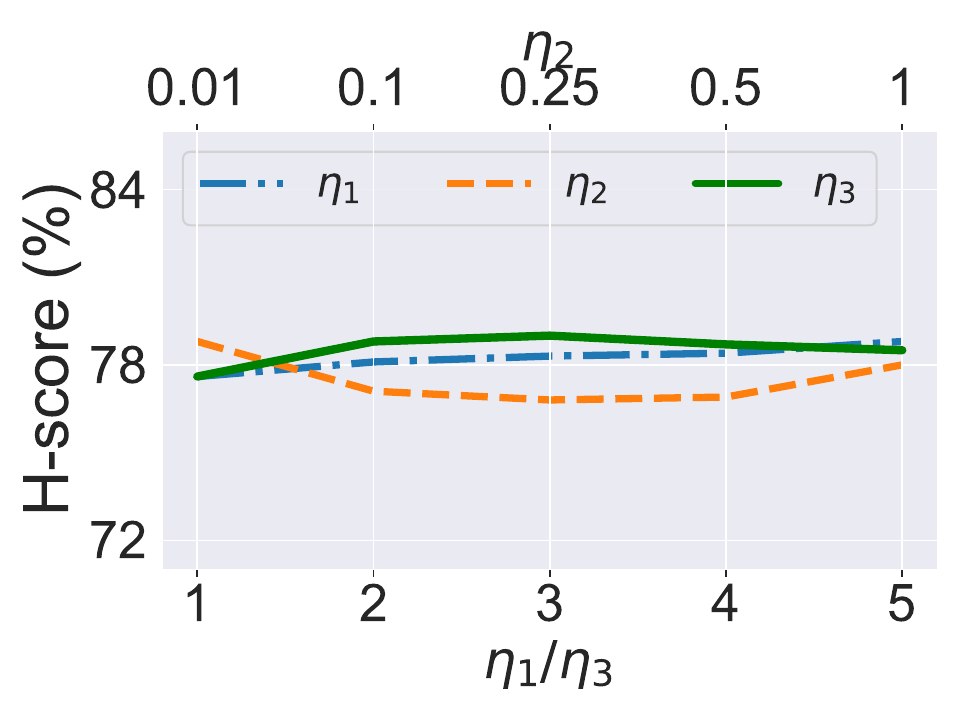} 
			\label{fig:eta}}
		\subfigure[]{ \includegraphics[width=0.223\textwidth]{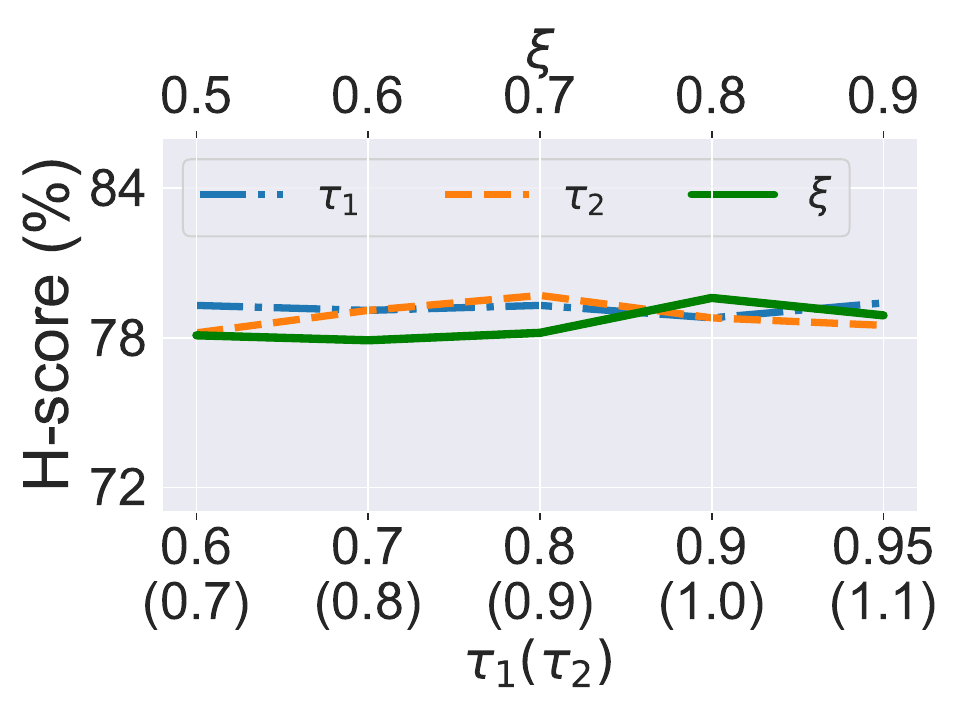} 
			\label{fig:tau}}
		\caption{ Sensitivity to hyper-parameters (a) $\eta_1$, $\eta_2$ and $\eta_3$ in Eqn.~(\ref{eq: total}), ~(b) $\tau_1$, $\tau_2$ in Eqn.~(\ref{eq: alpha}) and threshold $\xi$. All results are for the OPDA setting in task C$\rightarrow$A.}
		\label{fig:hyper-parameters}
	\end{figure}

 	\begin{figure}[t]
		\centering
		\subfigure[A$\rightarrow$P]{ \includegraphics[width=0.223\textwidth]{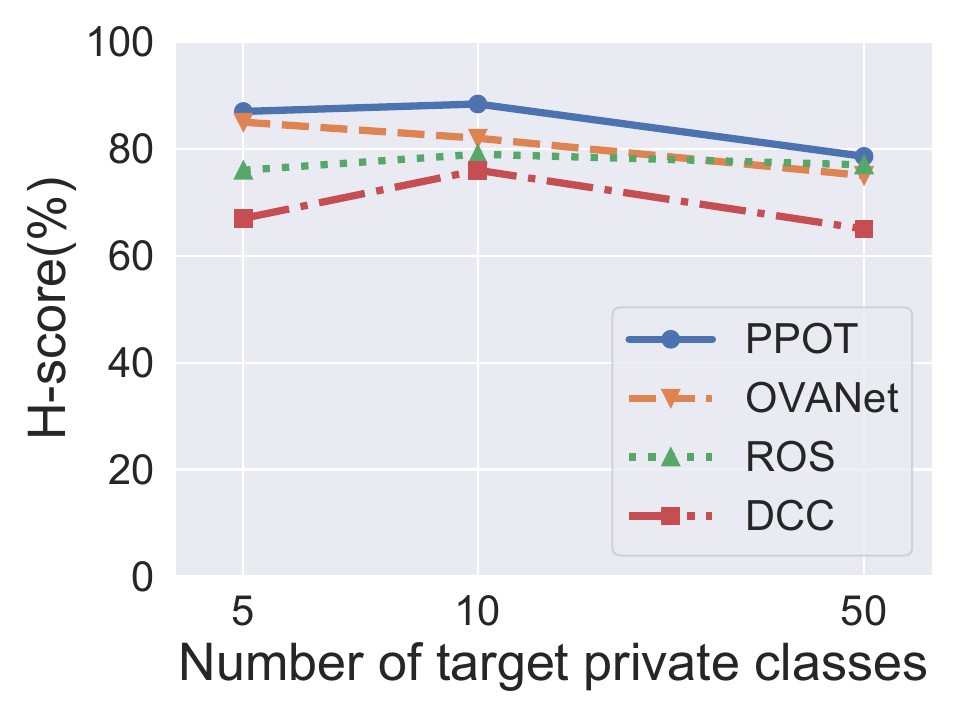} 
			\label{fig:classesA2P}}
		\subfigure[P$\rightarrow$R]{ \includegraphics[width=0.223\textwidth]{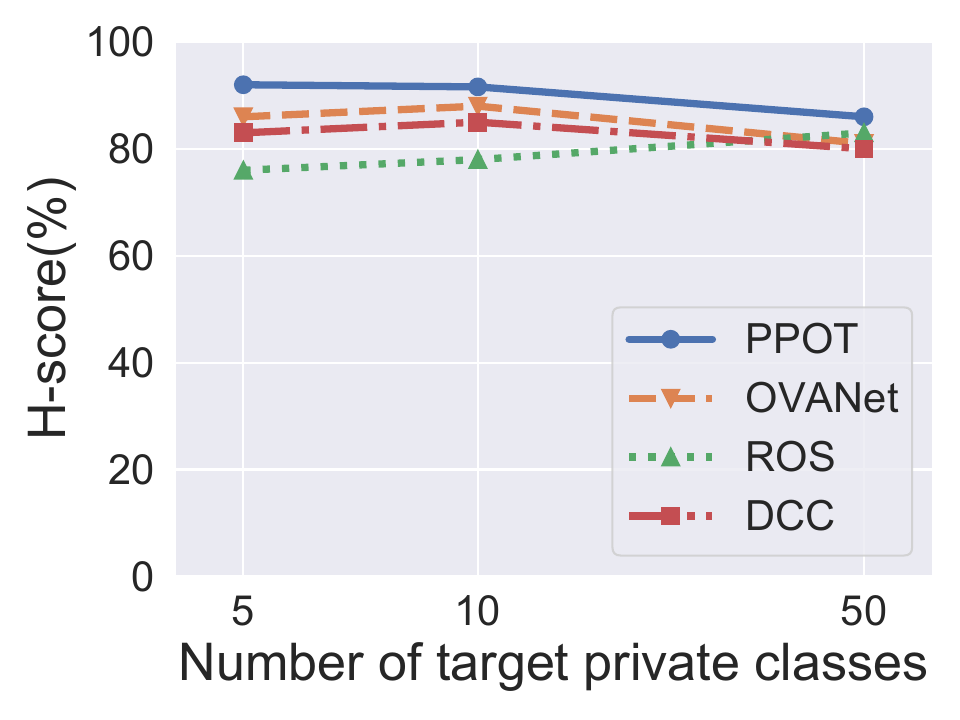} 
			\label{fig:classesP2R}}
		\caption{H-score curves of different methods with varying number of target private classes for OPDA tasks A$\rightarrow$P and P$\rightarrow$R.}
		\label{fig:class}
	\end{figure}
 
	\subsubsection{H-score with varying number of target private classes.} We evaluate our method with different numbers of target private classes. Results in A$\rightarrow$P and P$\rightarrow$R tasks are shown in Fig.~\ref{fig:class}, our method outperforms other baselines in all cases. It shows that our method is effective for OPDA with respect to different numbers of target domain private classes, and the performance marginally decreases with the increase of the number of target domain private classes.
	
	\section{Conclusion}
	In this paper, we propose to formulate the universal domain adaption (UniDA) as a partial optimal transport problem in deep learning framework. We propose a novel mini-batch based prototypical partial OT (m-PPOT) model for UniDA task, which is based on minimizing mini-batch prototypical partial optimal transport between two domain samples. We also introduce reweighting strategy based on the transport plan in UniDA. Experiments on four benchmarks show the effectiveness of our method for UniDA tasks including OPDA, PDA, OSDA settings.
 In the future, we plan to further theoretically analyze the mini-batch based m-PPOT, and apply it to more applications requiring partial alignments in deep learning framework.

    \bibliography{aaai23}

\begin{thebibliography}{51}
\providecommand{\natexlab}[1]{#1}

\bibitem[{Arjovsky, Chintala, and Bottou(2017)}]{arjovsky2017wasserstein}
Arjovsky, M.; Chintala, S.; and Bottou, L. 2017.
\newblock Wasserstein generative adversarial networks.
\newblock In \emph{ICML}.

\bibitem[{Ben-David et~al.(2010)Ben-David, Blitzer, Crammer, Kulesza, Pereira, and Vaughan}]{ben2010theory}
Ben-David, S.; Blitzer, J.; Crammer, K.; Kulesza, A.; Pereira, F.; and Vaughan, J.~W. 2010.
\newblock A theory of learning from different domains.
\newblock \emph{ML}, 79(1): 151--175.

\bibitem[{Ben-David et~al.(2006)Ben-David, Blitzer, Crammer, and Pereira}]{ben2006analysis}
Ben-David, S.; Blitzer, J.; Crammer, K.; and Pereira, F. 2006.
\newblock Analysis of representations for domain adaptation.
\newblock In \emph{NeurIPS}.

\bibitem[{Benamou et~al.(2015)Benamou, Carlier, Cuturi, Nenna, and Peyr{\'e}}]{benamou2015iterative}
Benamou, J.-D.; Carlier, G.; Cuturi, M.; Nenna, L.; and Peyr{\'e}, G. 2015.
\newblock Iterative Bregman projections for regularized transportation problems.
\newblock \emph{SISC}, 37(2): A1111--A1138.

\bibitem[{Bucci, Loghmani, and Tommasi(2020)}]{bucci2020effectiveness}
Bucci, S.; Loghmani, M.~R.; and Tommasi, T. 2020.
\newblock On the effectiveness of image rotation for open set domain adaptation.
\newblock In \emph{ECCV}.

\bibitem[{Caffarelli and McCann(2010)}]{caffarelli2010free}
Caffarelli, L.~A.; and McCann, R.~J. 2010.
\newblock Free boundaries in optimal transport and Monge-Ampere obstacle problems.
\newblock \emph{Annals of Mathematics}, 673--730.

\bibitem[{Cao et~al.(2018{\natexlab{a}})Cao, Long, Wang, and Jordan}]{cao2018part}
Cao, Z.; Long, M.; Wang, J.; and Jordan, M.~I. 2018{\natexlab{a}}.
\newblock Partial transfer learning with selective adversarial networks.
\newblock In \emph{CVPR}.

\bibitem[{Cao et~al.(2018{\natexlab{b}})Cao, Ma, Long, and Wang}]{cao2018partial}
Cao, Z.; Ma, L.; Long, M.; and Wang, J. 2018{\natexlab{b}}.
\newblock Partial adversarial domain adaptation.
\newblock In \emph{ECCV}.

\bibitem[{Cao et~al.(2019)Cao, You, Long, Wang, and Yang}]{cao2019learning}
Cao, Z.; You, K.; Long, M.; Wang, J.; and Yang, Q. 2019.
\newblock Learning to transfer examples for partial domain adaptation.
\newblock In \emph{CVPR}.

\bibitem[{Chen et~al.(2022)Chen, Lou, He, Bai, and Deng}]{chen2022geometric}
Chen, L.; Lou, Y.; He, J.; Bai, T.; and Deng, M. 2022.
\newblock Geometric Anchor Correspondence Mining With Uncertainty Modeling for Universal Domain Adaptation.
\newblock In \emph{CVPR}.

\bibitem[{Chen et~al.(2020)Chen, Fan, Girshick, and He}]{chen2020improved}
Chen, X.; Fan, H.; Girshick, R.; and He, K. 2020.
\newblock Improved baselines with momentum contrastive learning.
\newblock \emph{arXiv preprint arXiv:2003.04297}.

\bibitem[{Courty et~al.(2017)Courty, Flamary, Habrard, and Rakotomamonjy}]{courty2017joint}
Courty, N.; Flamary, R.; Habrard, A.; and Rakotomamonjy, A. 2017.
\newblock Joint distribution optimal transportation for domain adaptation.
\newblock In \emph{NeurIPS}.

\bibitem[{Cuturi(2013)}]{cuturi2013sinkhorn}
Cuturi, M. 2013.
\newblock Sinkhorn distances: Lightspeed computation of optimal transport.
\newblock In \emph{NeurIPS}.

\bibitem[{Damodaran et~al.(2018)Damodaran, Kellenberger, Flamary, Tuia, and Courty}]{damodaran2018deepjdot}
Damodaran, B.~B.; Kellenberger, B.; Flamary, R.; Tuia, D.; and Courty, N. 2018.
\newblock Deepjdot: Deep joint distribution optimal transport for unsupervised domain adaptation.
\newblock In \emph{ECCV}.

\bibitem[{Fatras et~al.(2021)Fatras, S{\'e}journ{\'e}, Flamary, and Courty}]{fatras2021unbalanced}
Fatras, K.; S{\'e}journ{\'e}, T.; Flamary, R.; and Courty, N. 2021.
\newblock Unbalanced minibatch optimal transport; applications to domain adaptation.
\newblock In \emph{ICML}.

\bibitem[{Figalli(2010)}]{figalli2010optimal}
Figalli, A. 2010.
\newblock The optimal partial transport problem.
\newblock \emph{Archive for Rational Mechanics and Analysis}, 195(2): 533--560.

\bibitem[{Flamary et~al.(2016)Flamary, Courty, Tuia, and Rakotomamonjy}]{flamary2016optimal}
Flamary, R.; Courty, N.; Tuia, D.; and Rakotomamonjy, A. 2016.
\newblock Optimal transport for domain adaptation.
\newblock \emph{TPAMI}, 1.

\bibitem[{Fu et~al.(2020)Fu, Cao, Long, and Wang}]{fu2020learning}
Fu, B.; Cao, Z.; Long, M.; and Wang, J. 2020.
\newblock Learning to detect open classes for universal domain adaptation.
\newblock In \emph{ECCV}.

\bibitem[{Ganin and Lempitsky(2015)}]{ganin2015}
Ganin, Y.; and Lempitsky, V. 2015.
\newblock Unsupervised domain adaptation by backpropagation.
\newblock In \emph{ICML}.

\bibitem[{Ge et~al.(2021)Ge, Liu, Li, Yoshie, and Sun}]{ge2021ota}
Ge, Z.; Liu, S.; Li, Z.; Yoshie, O.; and Sun, J. 2021.
\newblock Ota: Optimal transport assignment for object detection.
\newblock In \emph{CVPR}.

\bibitem[{Gu et~al.(2021)Gu, Yu, Sun, Xu et~al.}]{gu2021adversarial}
Gu, X.; Yu, X.; Sun, J.; Xu, Z.; et~al. 2021.
\newblock Adversarial Reweighting for Partial Domain Adaptation.
\newblock In \emph{NeurIPS}.

\bibitem[{He et~al.(2016)He, Zhang, Ren, and Sun}]{he2016deep}
He, K.; Zhang, X.; Ren, S.; and Sun, J. 2016.
\newblock Deep residual learning for image recognition.
\newblock In \emph{CVPR}.

\bibitem[{Ho et~al.(2017)Ho, Nguyen, Yurochkin, Bui, Huynh, and Phung}]{ho2017multilevel}
Ho, N.; Nguyen, X.; Yurochkin, M.; Bui, H.~H.; Huynh, V.; and Phung, D. 2017.
\newblock Multilevel clustering via Wasserstein means.
\newblock In \emph{ICML}.

\bibitem[{Kantorovitch(1958)}]{kantorovitch1958}
Kantorovitch, L. 1958.
\newblock On the translocation of masses.
\newblock \emph{Management Science}, 5(1): 1--4.

\bibitem[{Krizhevsky, Sutskever, and Hinton(2012)}]{krizhevsky2012imagenet}
Krizhevsky, A.; Sutskever, I.; and Hinton, G.~E. 2012.
\newblock Imagenet classification with deep convolutional neural networks.
\newblock In \emph{NeurIPS}.

\bibitem[{Li et~al.(2021)Li, Kang, Zhu, Wei, and Yang}]{li2021domain}
Li, G.; Kang, G.; Zhu, Y.; Wei, Y.; and Yang, Y. 2021.
\newblock Domain consensus clustering for universal domain adaptation.
\newblock In \emph{CVPR}.

\bibitem[{Liu et~al.(2019)Liu, Cao, Long, Wang, and Yang}]{liu2019separate}
Liu, H.; Cao, Z.; Long, M.; Wang, J.; and Yang, Q. 2019.
\newblock Separate to adapt: Open set domain adaptation via progressive separation.
\newblock In \emph{CVPR}.

\bibitem[{Long et~al.(2015)Long, Cao, Wang, and Jordan}]{long2015learning}
Long, M.; Cao, Y.; Wang, J.; and Jordan, M. 2015.
\newblock Learning transferable features with deep adaptation networks.
\newblock In \emph{ICML}.

\bibitem[{Long et~al.(2018)Long, Cao, Wang, and Jordan}]{long2018conditional}
Long, M.; Cao, Z.; Wang, J.; and Jordan, M.~I. 2018.
\newblock Conditional adversarial domain adaptation.
\newblock In \emph{NeurIPS}.

\bibitem[{Nguyen et~al.(2022)Nguyen, Nguyen, Pham, Ho et~al.}]{nguyen2022improving}
Nguyen, K.; Nguyen, D.; Pham, T.; Ho, N.; et~al. 2022.
\newblock Improving mini-batch optimal transport via partial transportation.
\newblock In \emph{ICML}.

\bibitem[{Pan et~al.(2019)Pan, Yao, Li, Wang, Ngo, and Mei}]{pan2019transferrable}
Pan, Y.; Yao, T.; Li, Y.; Wang, Y.; Ngo, C.-W.; and Mei, T. 2019.
\newblock Transferrable prototypical networks for unsupervised domain adaptation.
\newblock In \emph{CVPR}.

\bibitem[{Panareda~Busto and Gall(2017)}]{panareda2017open}
Panareda~Busto, P.; and Gall, J. 2017.
\newblock Open set domain adaptation.
\newblock In \emph{ICCV}.

\bibitem[{Papyan, Han, and Donoho(2020)}]{papyan2020prevalence}
Papyan, V.; Han, X.; and Donoho, D.~L. 2020.
\newblock Prevalence of neural collapse during the terminal phase of deep learning training.
\newblock \emph{PNAS}, 117(40): 24652--24663.

\bibitem[{Paszke et~al.(2019)Paszke, Gross, Massa, Lerer, Bradbury, Chanan, Killeen, Lin, Gimelshein, Antiga et~al.}]{paszke2019pytorch}
Paszke, A.; Gross, S.; Massa, F.; Lerer, A.; Bradbury, J.; Chanan, G.; Killeen, T.; Lin, Z.; Gimelshein, N.; Antiga, L.; et~al. 2019.
\newblock Pytorch: An imperative style, high-performance deep learning library.
\newblock In \emph{NeurIPS}.

\bibitem[{Peng et~al.(2019)Peng, Bai, Xia, Huang, Saenko, and Wang}]{domainnet}
Peng, X.; Bai, Q.; Xia, X.; Huang, Z.; Saenko, K.; and Wang, B. 2019.
\newblock Moment matching for multi-source domain adaptation.
\newblock In \emph{ICCV}.

\bibitem[{Peng et~al.(2017)Peng, Usman, Kaushik, Hoffman, Wang, and Saenko}]{visda}
Peng, X.; Usman, B.; Kaushik, N.; Hoffman, J.; Wang, D.; and Saenko, K. 2017.
\newblock Visda: The visual domain adaptation challenge.
\newblock \emph{arXiv preprint arXiv:1710.06924}.

\bibitem[{Peyr{\'e}, Cuturi et~al.(2019)}]{peyre2019computational}
Peyr{\'e}, G.; Cuturi, M.; et~al. 2019.
\newblock Computational optimal transport: With applications to data science.
\newblock \emph{Foundations and Trends{\textregistered} in Machine Learning}, 11(5-6): 355--607.

\bibitem[{Saenko et~al.(2010)Saenko, Kulis, Fritz, and Darrell}]{office31}
Saenko, K.; Kulis, B.; Fritz, M.; and Darrell, T. 2010.
\newblock Adapting visual category models to new domains.
\newblock In \emph{ECCV}.

\bibitem[{Saito et~al.(2020)Saito, Kim, Sclaroff, and Saenko}]{saito2020universal}
Saito, K.; Kim, D.; Sclaroff, S.; and Saenko, K. 2020.
\newblock Universal domain adaptation through self supervision.
\newblock In \emph{NeurIPS}.

\bibitem[{Saito and Saenko(2021)}]{saito2021ovanet}
Saito, K.; and Saenko, K. 2021.
\newblock Ovanet: One-vs-all network for universal domain adaptation.
\newblock In \emph{ICCV}.

\bibitem[{Saito et~al.(2018)Saito, Yamamoto, Ushiku, and Harada}]{saito2018open}
Saito, K.; Yamamoto, S.; Ushiku, Y.; and Harada, T. 2018.
\newblock Open set domain adaptation by backpropagation.
\newblock In \emph{ECCV}.

\bibitem[{Shen et~al.(2022)Shen, Jones, Kumar, Xie, HaoChen, Ma, and Liang}]{shen2022connect}
Shen, K.; Jones, R.~M.; Kumar, A.; Xie, S.~M.; HaoChen, J.~Z.; Ma, T.; and Liang, P. 2022.
\newblock Connect, not collapse: Explaining contrastive learning for unsupervised domain adaptation.
\newblock In \emph{ICML}.

\bibitem[{Simonyan and Zisserman(2014)}]{simonyan2014very}
Simonyan, K.; and Zisserman, A. 2014.
\newblock Very deep convolutional networks for large-scale image recognition.
\newblock \emph{arXiv preprint arXiv:1409.1556}.

\bibitem[{Sun and Saenko(2016)}]{sun2016deep}
Sun, B.; and Saenko, K. 2016.
\newblock Deep coral: Correlation alignment for deep domain adaptation.
\newblock In \emph{ECCV}.

\bibitem[{Venkateswara et~al.(2017)Venkateswara, Eusebio, Chakraborty, and Panchanathan}]{officehome}
Venkateswara, H.; Eusebio, J.; Chakraborty, S.; and Panchanathan, S. 2017.
\newblock Deep hashing network for unsupervised domain adaptation.
\newblock In \emph{CVPR}.

\bibitem[{Villani(2009)}]{villani2009optimal}
Villani, C. 2009.
\newblock \emph{Optimal transport: old and new}, volume 338.
\newblock Springer.

\bibitem[{Xie et~al.(2018)Xie, Zheng, Chen, and Chen}]{xie2018learning}
Xie, S.; Zheng, Z.; Chen, L.; and Chen, C. 2018.
\newblock Learning semantic representations for unsupervised domain adaptation.
\newblock In \emph{ICML}.

\bibitem[{Xu et~al.(2021)Xu, Liu, Zhang, Cai, Wang, Liang, Ying, and Yin}]{xu2021joint}
Xu, R.; Liu, P.; Zhang, Y.; Cai, F.; Wang, J.; Liang, S.; Ying, H.; and Yin, J. 2021.
\newblock Joint partial optimal transport for open set domain adaptation.
\newblock In \emph{IJCAI}.

\bibitem[{You et~al.(2019)You, Long, Cao, Wang, and Jordan}]{you2019universal}
You, K.; Long, M.; Cao, Z.; Wang, J.; and Jordan, M.~I. 2019.
\newblock Universal domain adaptation.
\newblock In \emph{CVPR}.

\bibitem[{Zhang et~al.(2018)Zhang, Ding, Li, and Ogunbona}]{zhang2018importance}
Zhang, J.; Ding, Z.; Li, W.; and Ogunbona, P. 2018.
\newblock Importance weighted adversarial nets for partial domain adaptation.
\newblock In \emph{CVPR}.

\bibitem[{Zhang et~al.(2019)Zhang, Liu, Long, and Jordan}]{zhang2019bridging}
Zhang, Y.; Liu, T.; Long, M.; and Jordan, M. 2019.
\newblock Bridging theory and algorithm for domain adaptation.
\newblock In \emph{ICML}.

\end{thebibliography}

    \appendix
	\renewcommand{\thetable}{A-\arabic{table}}
\renewcommand{\theequation}{A-\arabic{equation}}
\renewcommand{\thefigure}{A-\arabic{table}}
\renewcommand{\thefigure}{A-\arabic{figure}}

\setcounter{proposition}{0}
\setcounter{figure}{0}
\setcounter{equation}{0}
\setcounter{table}{0}
\setcounter{thm}{0}

	\section{Mathematical Deductions}
	\subsection{Notation}
	We first recall some definitions in our paper. The source and target empirical distributions in feature space are denoted as $\bm{p} = \sum_{i=1}^{m}p_{i}\delta_{f(x_{i}^{s})}$ and $\bm{q} = \sum_{j=1}^{n}q_{j}\delta_{f(x_{j}^{t})}$ respectively, $\sum_{i=1}^{m}p_{i} =1$, $\sum_{j=1}^{n}q_{j} = 1$. $\bm{c} = \sum_{i=1}^{L}r_{i}\delta_{c_{i}}$ is the empirical distribution of source domain prototypes, and $r_{i} = \sum_{j: y_{j}=i}p_{j}$. With a slight abuse of notations, we denote the vector of data mass as $\bm{p} = (p_{1}, p_{2}, ... , p_{m})^{\top}$, $\bm{q} = (q_{1}, q_{2}, ... , q_{n})^{\top}$ and $\bm{c} = (r_{1}, r_{2}, ... , r_{L})^{\top}$. In this paper, we set $q_j = \frac{1}{n}$ for any $0 \leqslant j\leqslant n$.
	
	The elements of cost matrix $C^{\bm{p},\bm{q}}, C^{\bm{p},\bm{c}}, C^{\bm{c},\bm{q}}$ are defined as $C^{\bm{p},\bm{q}}_{ij} = d(f(x_{i}^{s}), f(x_{j}^{t})), C^{\bm{p},\bm{c}}_{ik} = d(f(x_{i}^{s}), c_k)), C^{\bm{c},\bm{q}}_{kj} = d(c_k, f(x_{j}^{t}))$, respectively. Here $d$ is the $L_2$ distance.
	
	Furthermore, the definitions of  $\textup{PPOT}^{\alpha}(\bm{p}, \bm{q})$ and $\textup{m-PPOT}^{\alpha}_{\mathcal{B}}(\bm{p}. \bm{q})$ are shown as follows:
	\begin{equation}
		\textup{PPOT}^{\alpha}(\bm{p}, \bm{q}) \triangleq \textup{POT}^{\alpha}(\bm{c}, \bm{q}) = \min_{\pi \in \Pi^{\alpha}(\bm{c},\bm{q})}\langle \pi , C^{\bm{c},\bm{q}} \rangle_F,
	\end{equation}
	\begin{equation}
		\text{m-PPOT}^{\alpha}_{\mathcal{B}}(\bm{p}, \bm{q}) \triangleq \frac{1}{k}\sum_{i=1}^{k}\text{POT}^{\alpha}(\bm{c}, \bm{q}_{\mathcal{B}_{i}}),~~ \mathcal{B} \in \Gamma,
	\end{equation}
	where $\mathcal{B}_{i}$ is the $i$-th index set of $b$ random target samples and their corresponding empirical distribution in feature space is denoted as $\bm{q}_{\mathcal{B}_{i}}$. $\mathcal{B} \triangleq \{\mathcal{B}_{i}\}_{i=1}^{k}$, satisfying $\mathcal{B}_{i} \bigcap \mathcal{B}_{j} = \emptyset$ and $\bigcup\limits_{i=1}^{k}\mathcal{B}_{i} = \{1,2,...,n\}$. $b\mid n$, $k =\frac{n}{b}$.
	\subsection{Proof of proposition 1}
	\begin{proposition}\label{prop: transport}
		Let $\pi^{\alpha}_{i}$ be the optimal transportation of $i$-th batch of $\text{m-PPOT}^{\alpha}_{\mathcal{B}}(\bm{p}, \bm{q})$. We extend $\pi_{i}^{\alpha}$ to a $L \times n$ matrix $\Pi_{i}^{\alpha}$  that pads zero entries to the column whose index does not belong to $\mathcal{B}_{i}$, then we have
		\begin{equation} \label{eq: transportation}		
			\dfrac{1}{k}\sum_{i=1}^{k}\Pi^{\alpha}_{i} \in \Pi^{\alpha}(\bm{c},\bm{q})			
		\end{equation}
		and
		\begin{equation} \label{eq: value}		
			\textup{PPOT}^{\alpha}(\bm{p}, \bm{q}) \leqslant \textup{m-PPOT}^{\alpha}_{\mathcal{B}}(\bm{p}, \bm{q}).				
		\end{equation}
	\end{proposition}
	\begin{proof}
		The proof of  
		\begin{equation}			
			\dfrac{1}{k}\sum_{i=1}^{k}\Pi^{\alpha}_{i} \in \Pi^{\alpha}(\bm{c},\bm{q})			
		\end{equation}
		is equivalent to proving
		\begin{equation} \label{eq: prop1}
			\begin{aligned}			
				&(\dfrac{1}{k}\sum_{i=1}^{k}\Pi^{\alpha}_{i})\mathbb{1}_{n} \leqslant \bm{c}, \\
				&(\dfrac{1}{k}\sum_{i=1}^{k}\Pi^{\alpha}_{i})^{\top}\mathbb{1}_{L} \leqslant \bm{q}, \\
				&\mathbb{1}_{L}^{\top}(\dfrac{1}{k}\sum_{i=1}^{k}\Pi^{\alpha}_{i})\mathbb{1}_{n} = \alpha,
			\end{aligned}		
		\end{equation}
		according to the definition of $\Pi^{\alpha}(\bm{c},\bm{q})$.
		Note that $\pi^{\alpha}_{i} \in \Pi^{\alpha}(\bm{c}, \bm{q}_{\mathcal{B}_i})$ satisfies 
		\begin{equation}
			\pi^{\alpha}_{i}\mathbb{1}_{b} \leqslant \bm{c}, ~~(\pi^{\alpha}_{i})^{\top}\mathbb{1}_{L} \leqslant \bm{q}_{\mathcal{B}_i}, ~~\mathbb{1}_{L}^{\top}\pi^{\alpha}_{i}\mathbb{1}_{b} = \alpha.
		\end{equation}
		Combining with the definition of $\Pi^{\alpha}_{i}$, we have
		\begin{equation} 
			\begin{aligned}			
				&\Pi^{\alpha}_{i}\mathbb{1}_{n} = \pi^{\alpha}_{i}\mathbb{1}_{b} \leqslant \bm{c}, \\
				&(\Pi^{\alpha}_{i})^{\top}\mathbb{1}_{L}   \leqslant \bar{\bm{q}}_{\mathcal{B}_i} = \bm{m}^{i} \odot \bm{q}, \\
				&\mathbb{1}_{L}^{\top}\Pi^{\alpha}_{i}\mathbb{1}_{n} = \mathbb{1}_{L}^{\top}\pi^{\alpha}_{i}\mathbb{1}_{b}=\alpha \\
			\end{aligned}		
		\end{equation}
		where $\bar{\bm{q}}_{\mathcal{B}_i} \in \mathbb{R}^{n}$ is the extension of $\bm{q}_{\mathcal{B}_i}$ by padding zero entries to the dimension whose index does not belong to $\mathcal{B}_i$, $\odot$ corresponds to entry-wise product and $\bm{m}^{i}$ is a $n$ dimensional vector with element satisfying that
		\begin{equation} \label{eq: m}
			m_{j}^i=\begin{cases} \dfrac{n}{b}=k, & \text{if} ~~ j \in \mathcal{B}_{i},\\
				\quad 0, & \text{otherwise}.\\
			\end{cases}
		\end{equation}
		We have
		\begin{equation} \label{eq: margin}
			\begin{aligned}			
				&(\dfrac{1}{k}\sum_{i=1}^{k}\Pi^{\alpha}_{i})\mathbb{1}_{n} = \dfrac{1}{k}\sum_{i=1}^{k}(\Pi^{\alpha}_{i}\mathbb{1}_{n}) \leqslant \bm{c}, \\
				&(\dfrac{1}{k}\sum_{i=1}^{k}\Pi^{\alpha}_{i})^{\top}\mathbb{1}_{L} \leqslant (\dfrac{1}{k}\sum_{i=1}^{k} \bm{m}^{i}) \odot \bm{q}, \\
				&\mathbb{1}_{L}^{\top}(\dfrac{1}{k}\sum_{i=1}^{k}\Pi^{\alpha}_{i})\mathbb{1}_{n} = \dfrac{1}{k}\sum_{i=1}^{k}(\mathbb{1}_{L}^{\top}\Pi^{\alpha}_{i}\mathbb{1}_{n}) =\alpha,
			\end{aligned}		
		\end{equation}
		Combining (\ref{eq: m}) with the conditions $\mathcal{B}_{i} \bigcap \mathcal{B}_{j} = \emptyset$ and $\bigcup\limits_{i=1}^{k}\mathcal{B}_{i} = \{1,2,...,n\}$, we find that
		\begin{equation}
			(\sum_{i=1}^{k} \bm{m}^{i})_{j} = k\sum_{i=1}^{k} \bm{1}(j \in \mathcal{B}_{i}) = k,
		\end{equation}
		it means that
		\begin{equation}
			\dfrac{1}{k}\sum_{i=1}^{k} \bm{m}^{i} = \mathbb{1}_n.
		\end{equation}
		Therefore Eqn.~(\ref{eq: prop1}) has been proved. So far, we have proved Eqn.(\ref{eq: transportation}), which means that the following inequality holds:
		\begin{equation}
			\begin{aligned}
				&\textup{PPOT}^{\alpha}(\bm{p}, \bm{q}) \leqslant \langle \dfrac{1}{k}\sum_{i=1}^{k}\Pi^{\alpha}_{i} , C^{\bm{c},\bm{q}} \rangle_F \\
				&= \dfrac{1}{k}\sum_{i=1}^{k}\langle\Pi^{\alpha}_{i} , C^{\bm{c},\bm{q}} \rangle_F\\
				&= \dfrac{1}{k}\sum_{i=1}^{k} \text{POT}^{\alpha}(\bm{c}, \bm{q}_{\mathcal{B}_{i}})\\
				&=\textup{m-PPOT}^{\alpha}_{\mathcal{B}}(\bm{p}, \bm{q}).
			\end{aligned}
		\end{equation}
	\end{proof}

	\subsection{Proof of theorem 1}
	\begin{thm}\label{thm: thmA2}
		Considering two distributions $\bm{p}$ and $\bm{q}$, the distance between $f(x_{i}^{s})$ and corresponding prototype $c_{y_{i}}$ is denoted as $d_{i} \triangleq d(f(x_{i}^{s}), c_{y_{i}})$. The row sum of the optimal transportation of $\text{PPOT}^{\alpha}(\bm{p}, \bm{q})$ is denoted as $\bm{w} = (w_1, w_2, ... , w_L )^{\top}$, $r_{i} = \sum_{j: y_{j}=i}p_{j}$. Then we have 
		\begin{equation}
			\textup{POT}^{\alpha}(\bm{p}, \bm{q}) \leqslant \sum_{i=1}^{m} \dfrac{w_{y_i}}{r_{y_i}} p_i d_{i} + \textup{m-PPOT}^{\alpha}_{\mathcal{B}}(\bm{p}, \bm{q}).
		\end{equation}
	\end{thm}
	The proof of theorem 1 can be decomposed by two Lemmas.
	\begin{lemma} \label{lm: lm1}
		Consider three distributions $\bm{p}, \bm{q}$ and $\bm{c}$, we denote the optimal transportation of $\textup{PPOT}^{\alpha}(\bm{p}, \bm{q})$ as $\pi^{\bm{c}, \bm{q}} \in \mathbb{R}^{L \times n}$, and denote the row sum of $\pi^{\bm{c}, \bm{q}}$ as $\bm{w} \triangleq \pi^{\bm{c}, \bm{q}} \mathbb{1}_n \leqslant \bm{c}$, $\bm{w} = (w_1, w_2, \dots , w_L)$. We define an empirical distribution $\bar{\bm{c}} = \sum\limits_{i=1}^{L} w_i \delta_{c_i}$. 
		Then, we have 
		\begin{equation}
			\textup{POT}^{\alpha}(\bm{p}, \bm{q}) \leqslant \textup{POT}^{\alpha}(\bm{p}, \bar{\bm{c}}) + \textup{PPOT}^{\alpha}(\bm{p}, \bm{q}).
		\end{equation}
	\end{lemma}
	\begin{proof}
		Let $\pi^{\bm{p}, \bm{c}} \in \mathbb{R}^{m \times L}$ be the optimal transportation matrix of $\textup{POT}^{\alpha}(\bm{p}, \bar{\bm{c}})$. Obviously, due to $\Vert \bm{w}\Vert_{1} = \alpha$, $\textup{POT}^{\alpha}(\bm{p}, \bar{\bm{c}})$ satisfies
		\begin{equation}
			\textup{POT}^{\alpha}(\bm{p}, \bar{\bm{c}}) = \min_{\pi \in \bar{\Pi}^{\alpha}(\bm{p}, \bar{\bm{c}})}  \langle \pi , C^{\bm{p},\bm{c}} \rangle_F,
		\end{equation}
		$\bar{\Pi}^{\alpha}(\bm{p}, \bar{\bm{c}}) = \{ \pi \mid \pi \mathbb{1}_{L} \leqslant \bm{p}, \pi^{\top}\mathbb{1}_m = \bm{w}, \mathbb{1}_m^{\top}\pi \mathbb{1}_L = \alpha \}$.
		
		Therefore, $\pi^{\bm{p}, \bm{c}}$ and $\pi^{\bm{c}, \bm{q}}$ satisfy these equations:
		\begin{align}
			&\pi^{\bm{p}, \bm{c}} \mathbb{1}_L \leqslant \bm{p}, \quad \mathbb{1}_m^{\top} \pi^{\bm{p}, \bm{c}} = \bm{w}^{\top}, \label{eq: pi1_cond}\\
			&\pi^{\bm{c}, \bm{q}} \mathbb{1}_n = \bm{w},\quad \mathbb{1}_L^{\top} \pi^{\bm{c}, \bm{q}}  \leqslant \bm{q}^{\top}, \label{eq: pi2_cond}
		\end{align}
		
		and then we define a $m \times n$ dimensional matrix $S$ as:
		\begin{equation*}
			S \triangleq \pi^{\bm{p}, \bm{c}}  D ~ \pi^{\bm{c}, \bm{q}},
		\end{equation*}
		where 
		\begin{equation*}
			D = \left(
			\begin{array}{cccc}
				\frac{1}{w_1} &  &  & \\
				& \frac{1}{w_2} &  & \\
				&  & \ddots & \\
				&  &  & \frac{1}{w_L}\\
			\end{array} \right).
		\end{equation*}
		Notice that $S \in \Pi^{\alpha}(\bm{p}, \bm{q})$ because according to Eqns.(\ref{eq: pi1_cond}) and (\ref{eq: pi2_cond}), we have
		\begin{equation*}
			\begin{aligned}
				S \mathbb{1}_n &= \pi^{\bm{p}, \bm{c}}  D ~ \pi^{\bm{c}, \bm{q}}\mathbb{1}_n = \pi^{\bm{p}, \bm{c}}  D ~\bm{w}
				= \pi^{\bm{p}, \bm{c}} \mathbb{1}_L \leqslant \bm{p}, \\
				\mathbb{1}_m^{\top} S &= \mathbb{1}_m^{\top} \pi^{\bm{p}, \bm{c}}  D ~ \pi^{\bm{c}, \bm{q}} = \bm{w}^{\top}D ~ \pi^{\bm{c}, \bm{q}} = \mathbb{1}_L^{\top}\pi^{\bm{c}, \bm{q}} \leqslant \bm{q}^{\top}.
			\end{aligned}
		\end{equation*}
		It means that $S$ is a transport plan of $\textup{POT}^{\alpha}(\bm{p}, \bm{q})$, so
		\begin{equation*}
			\begin{aligned}
				&\textup{POT}^{\alpha}(\bm{p}, \bm{q}) \leqslant \langle S , C^{\bm{p},\bm{q}} \rangle_F = \sum_{i=1}^{m}
				\sum_{j=1}^{n}C^{\bm{p},\bm{q}}_{ij}S_{ij} \\
				&= \sum_{i=1}^{m}\sum_{j=1}^{n}\sum_{k=1}^{L}\frac{1}{w_k}C^{\bm{p},\bm{q}}_{ij}\pi^{\bm{p}, \bm{c}}_{ik}
				\pi^{\bm{c}, \bm{q}}_{kj} \\
				&= \sum_{i,j,k}\frac{1}{w_k}d(f(x_{i}^{s}),f(x_{j}^{t}) )\pi^{\bm{p}, \bm{c}}_{ik}
				\pi^{\bm{c}, \bm{q}}_{kj}\\
				&\leqslant \sum_{i,j,k}\frac{1}{w_k}d(f(x_{i}^{s}),c_{k})\pi^{\bm{p}, \bm{c}}_{ik}
				\pi^{\bm{c}, \bm{q}}_{kj}  + \sum_{i,j,k}\frac{1}{w_k}d(c_{k},f(x_{j}^{t}) )\pi^{\bm{p}, \bm{c}}_{ik}
				\pi^{\bm{c}, \bm{q}}_{kj}\\
				&= \sum_{i,j,k}\frac{1}{w_k}C^{\bm{p},\bm{c}}_{ik}\pi^{\bm{p}, \bm{c}}_{ik}
				\pi^{\bm{c}, \bm{q}}_{kj}  + \sum_{i,j,k}\frac{1}{w_k}C^{\bm{c},\bm{q}}_{kj}\pi^{\bm{p}, \bm{c}}_{ik}
				\pi^{\bm{c}, \bm{q}}_{kj}\\
				&= \sum_{i,k}\frac{1}{w_k}C^{\bm{p},\bm{c}}_{ik}\pi^{\bm{p}, \bm{c}}_{ik}
				\sum_{j}\pi^{\bm{c}, \bm{q}}_{kj}  + \sum_{k,j}\frac{1}{w_k}C^{\bm{c},\bm{q}}_{kj}\pi^{\bm{c}, \bm{q}}_{kj}\sum_{i}\pi^{\bm{p}, \bm{c}}_{ik}\\
				&= \sum_{i,k}C^{\bm{p},\bm{c}}_{ik}\pi^{\bm{p}, \bm{c}}_{ik} + \sum_{k,j}C^{\bm{c},\bm{q}}_{kj}\pi^{\bm{c}, \bm{q}}_{kj}\\
				&=\textup{POT}^{\alpha}(\bm{p}, \bar{\bm{c}}) + \textup{PPOT}^{\alpha}(\bm{p}, \bm{q}).
			\end{aligned}
		\end{equation*}
	\end{proof}
	
	\begin{lemma} \label{lm: lm2}
		Consider two distributions $\bm{p}$ and $\bm{c}$, the definitions of $\textup{POT}^{\alpha}(\bm{p}, \bar{\bm{c}})$ and $\bm{w}$ follow Lemma \ref{lm: lm1}. The distance between $f(x_{i}^{s})$ and $c_{y_i}$ (\ie ~$C_{i,y_i}^{\bm{p},\bm{c}}$) is denoted as $d_{i}$, then we have
		\begin{equation}
			\textup{POT}^{\alpha}(\bm{p}, \bar{\bm{c}}) \leqslant \sum_{i=1}^{m}\dfrac{w_{y_i}}{r_{y_i}}p_i d_i.
		\end{equation}
	\end{lemma}
	
	\begin{proof}
		Let $\pi$ be a $m \times L$ dimensional matrix, which satisfies
		\begin{equation}
			\pi_{ij} = \begin{cases}
				\dfrac{w_j}{r_j}p_i, & \text{if} ~~ y_i = j, \\
				& \\
				\quad 0, & \text{otherwise}.
			\end{cases}
		\end{equation}
		Computing the row and column sums of $\pi$, we can find that
		\begin{align}
			\sum_{j=1}^L \pi_{ij} &= \dfrac{w_{y_i}}{r_{y_i}}p_i \leqslant p_i, \label{eq: row_sum}\\
			\sum_{i=1}^m \pi_{ij} &= \dfrac{w_j}{r_j}\sum_{i: y_i =j} p_i = w_j, \label{eq: col_sum}\\
			\sum_{i=1}^m\sum_{j=1}^L \pi_{ij} &= \sum_{j=1}^L w_j = \alpha. \label{eq: sum}
		\end{align}
		Eqns.(\ref{eq: row_sum}), (\ref{eq: col_sum}) and (\ref{eq: sum}) are equal to equations as follows:
		\begin{equation}
			\pi \mathbb{1}_L \leqslant \bm{p}, 
			~~\pi^{\top} \mathbb{1}_m = \bm{w}, 
			~~\mathbb{1}_m^{\top}\pi \mathbb{1}_L = \alpha,
		\end{equation}
		which means $\pi \in \bar{\Pi}^{\alpha}(\bm{p}, \bar{\bm{c}})$. Therefore, 
		\begin{equation}
			\textup{POT}^{\alpha}(\bm{p}, \bar{\bm{c}}) \leqslant \langle \pi , C^{\bm{p},\bm{c}} \rangle_F = \sum_{i=1}^{m}\dfrac{w_{y_i}}{r_{y_i}}p_i d_i.
		\end{equation}
	\end{proof}
	
	At last, we provide the proof of theorem \ref{thm: thmA2}.
	\begin{proof}
		Combining Lemmas \ref{lm: lm1}, \ref{lm: lm2} and Proposition 1
		\begin{equation*}
			\begin{aligned}
				&\textup{POT}^{\alpha}(\bm{p}, \bm{q}) \leqslant \textup{POT}^{\alpha}(\bm{p}, \bar{\bm{c}}) + \textup{PPOT}^{\alpha}(\bm{p}, \bm{q}),\\
				&\textup{POT}^{\alpha}(\bm{p}, \bar{\bm{c}}) \leqslant \sum_{i=1}^{m}\dfrac{w_{y_i}}{r_{y_i}}p_i d_i, \\
				&\textup{PPOT}^{\alpha}(\bm{p}, \bm{q}) \leqslant \textup{m-PPOT}^{\alpha}_{\mathcal{B}}(\bm{p}, \bm{q}),		
			\end{aligned}
		\end{equation*}
		we have 
		\begin{equation*}
			\textup{POT}^{\alpha}(\bm{p}, \bm{q}) \leqslant \sum_{i=1}^{m} \dfrac{w_{y_i}}{r_{y_i}} p_i d_{i} + \textup{m-PPOT}^{\alpha}_{\mathcal{B}}(\bm{p}, \bm{q}).
		\end{equation*}
	\end{proof}

	\begin{figure}[t]
		\centering
		\includegraphics[width=0.3\textwidth]{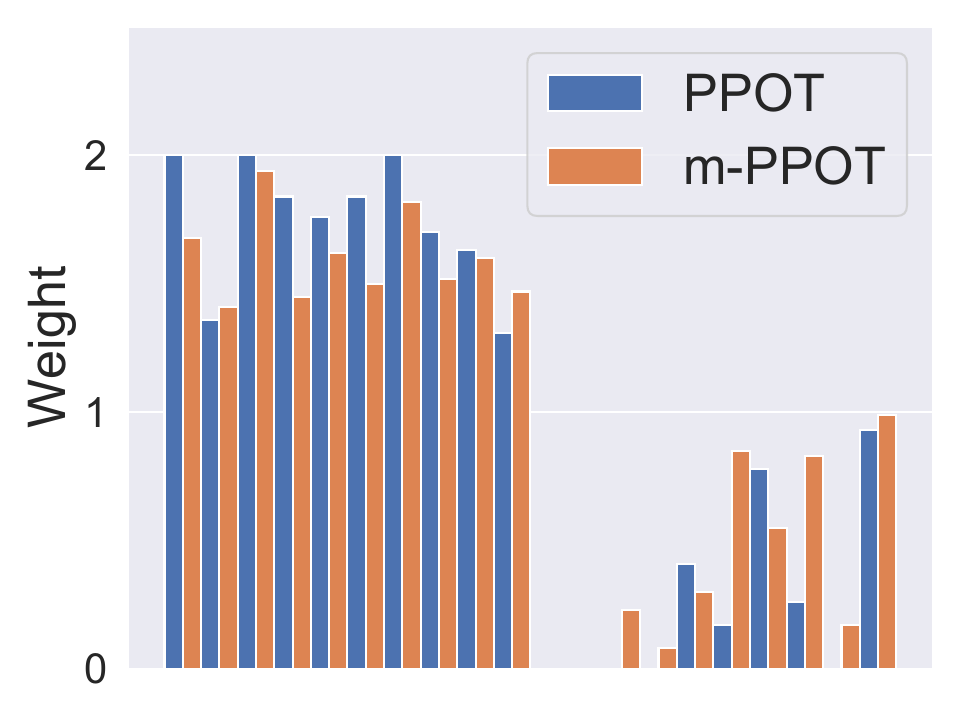} 
		\caption{Source class weight computed by the row sum of solutions of $\text{PPOT}^{\alpha}(\bm{p}, \bm{q})$ and $\text{m-PPOT}^{\alpha}(\bm{p}, \bm{q})$ for OPDA task D$\rightarrow$W.}
		\label{fig:s_weight}
	\end{figure}
	
	\begin{figure}[t]
		\centering 
		\includegraphics[width=0.3\textwidth]{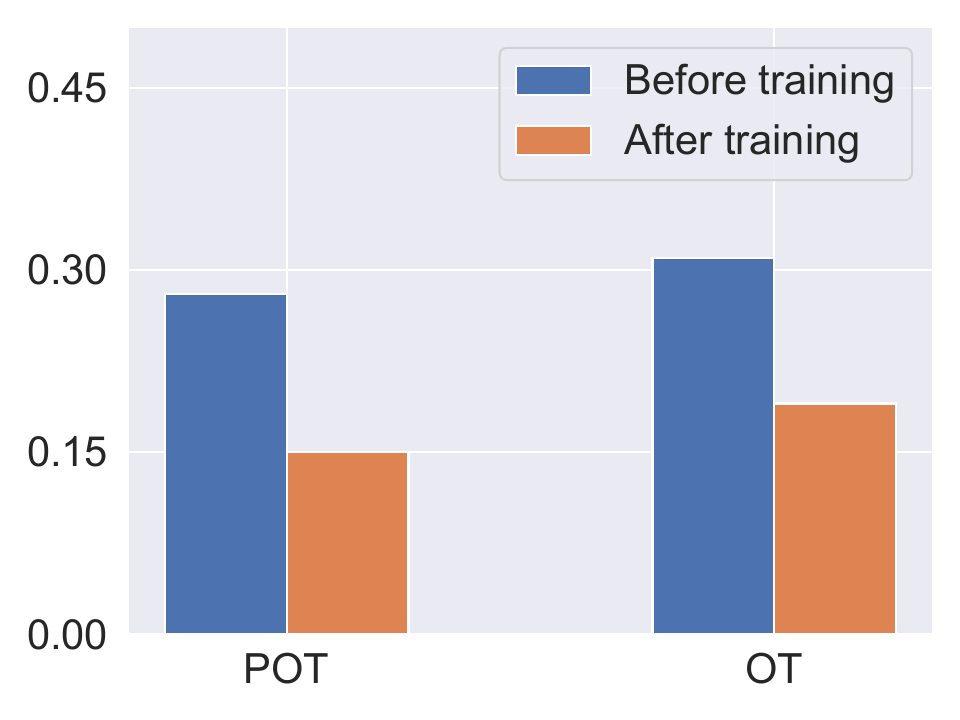} 
		\caption{The values of $\text{POT}^{\alpha}(\bm{p}, \bm{q})$ and $\text{OT}^{\alpha}(\bm{p}_c, \bm{q}_c)$ before and after training for OPDA task W$\rightarrow$D.}
		\label{fig:dist}
	\end{figure}
	
	\begin{figure}[!htbp]
		\centering 
		\includegraphics[width=0.3\textwidth]{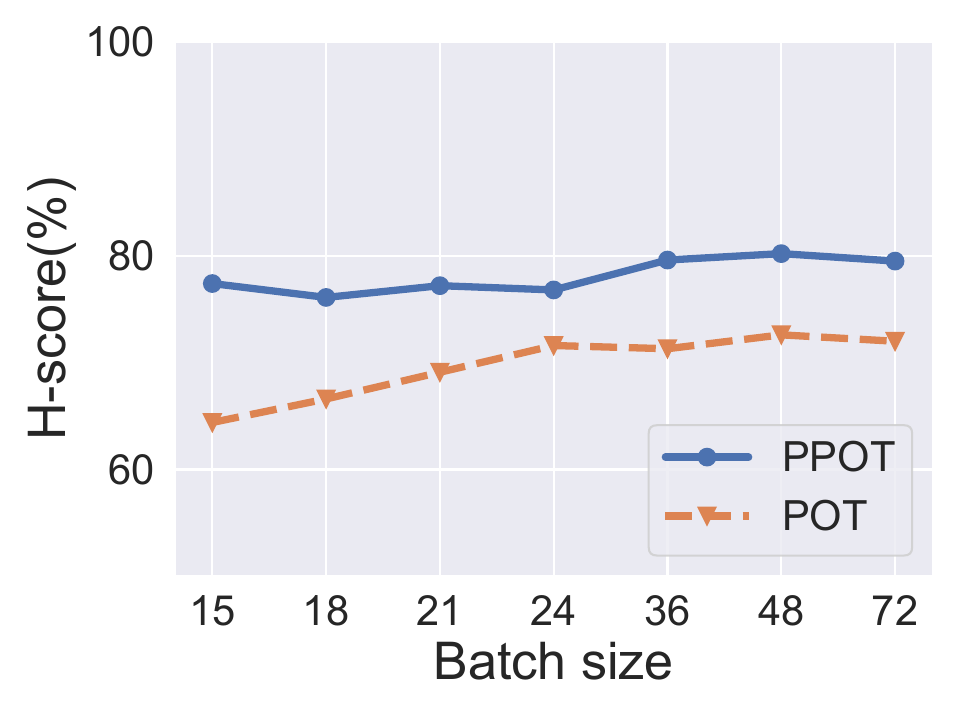} 
		\caption{H-score curves of PPOT and POT with varying batch size for OPDA task C$\rightarrow$A.}
		\label{fig:batch}
	\end{figure}

	\begin{figure}[!htbp]
		\centering 
		\includegraphics[width=0.3\textwidth]{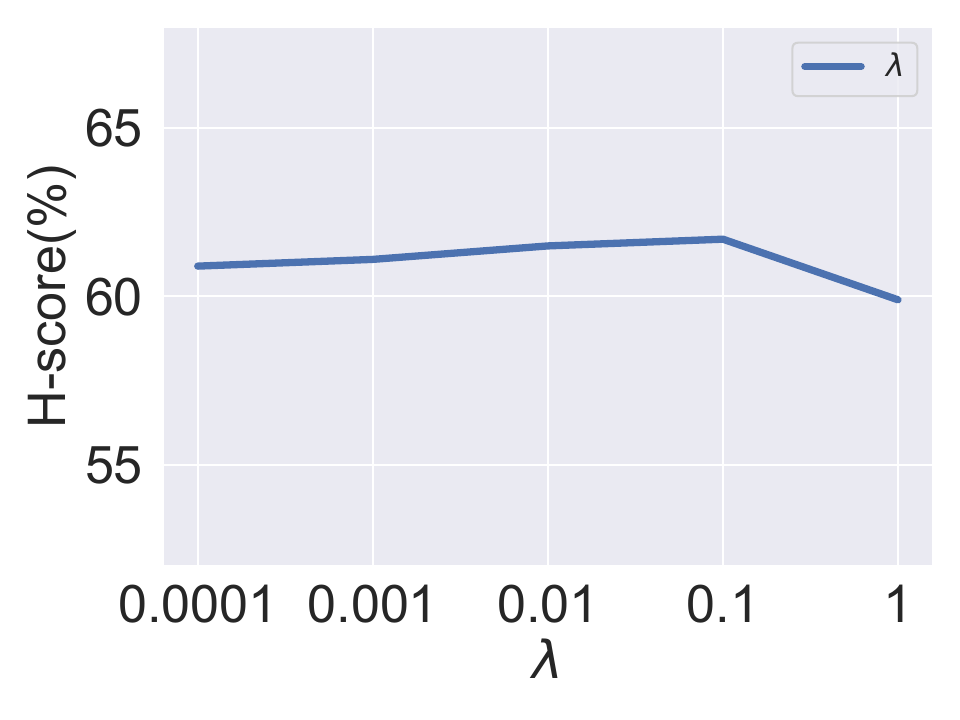} 
		\caption{Sensitivity to hyper-parameter $\lambda$ in Eqn.~(\ref{eq: c}). Results are for the OPDA setting in task P$\rightarrow$C.}
		\label{fig:lambda}
	\end{figure}

	\begin{table*}[!htbp]
		\centering
		\begin{tabular}{l|cccccc|>{\columncolor{mygray}}c}
			\toprule
			~ & \multicolumn{6}{c}{Office-31(10/10/11)} & ~\cellcolor{white}\\
			\midrule 
			Methods&A2D&A2W&D2A&D2W&W2A&W2D&Avg\\
			\midrule 
			UAN & 59.7 & 58.6 & 60.1 & 70.6 & 60.3 & 71.4 & 63.5\\
			CMU & 68.1 & 67.3 & 71.4 & 79.3 & 72.2 & 80.4 & 73.1\\
			DANCE & 79.6 & 75.8 & 82.9 & 90.9 & 77.6 & 87.1 & 82.3\\
			DCC & 88.5 & 78.5 & 70.2 & 79.3 & 75.9 & 88.6 & 80.2\\
			OVANet & 83.8 & 78.4 & 80.7 & \bf 95.9 & 82.7 & 95.5 & 86.2\\
			GATE & \bf 87.7 & 81.6 & 84.2 & 94.8 & 83.4 & 94.1 & 87.6\\
			\midrule
			\bf PPOT& 86.2 & \bf 87.0 & \bf 90.2 & 93.1 & \bf 90.2 & \bf 95.8 & \bf 90.4\\
			\bottomrule
		\end{tabular}
		\caption{H-score(\%) comparisons on Office-31 in OPDA setting.}	
		\label{tab:OPDA_OF}	
	\end{table*}
	\begin{table}[!htbp]
		\centering
		\begin{tabular}{l|c|c|c}
			\hline
			Methods & Office &  VisDA & OfficeHome\\
			\hline
			PPOT  & 82.3 & 69.3 & 49.6\\
			POT & 81.0 & 67.1 & 46.6\\	
			\hline
		\end{tabular}
		\caption{\textbf{Ablation study.} H-score(\%) comparison for UniDA on Office-31, OfficeHome and VisDA.} 
		\label{tab:ablation study}		
	\end{table}

	\begin{figure}[!htbp]
		\centering 
		\includegraphics[width=0.3\textwidth]{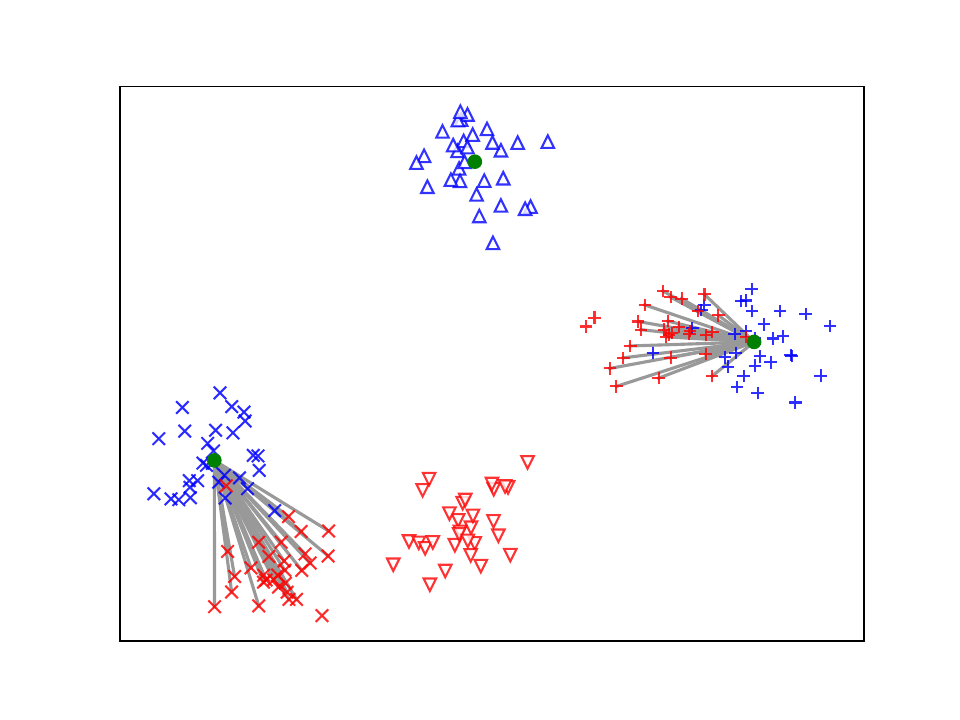} 
		\caption{Matching produced by PPOT in toy data. Different shapes mean different class, and the gray line is the matching computed by PPOT.}
		\label{fig:toydata}
	\end{figure}
	
	\section{Additional Empirical Analysis}
	\subsection{Comparison of the weight of the source prototypes derived from PPOT and m-PPOT}
	In the main paper,  we use $\bm{w}^s$, \ie, the row sum  of the optimal transport plan of $\text{m-PPOT}^{\alpha}(\bm{p}, \bm{q})$, to approximate $\bm{w}$ (\ie, the row sum of the solution of $\text{PPOT}^{\alpha}(\bm{p}, \bm{q})$) in $\mathcal{L}_{rce}$ for approximately minimizing the first term in the bound of Theorem 1. Figure~\ref{fig:s_weight} shows that the weight computed by the row sum of the optimal transport plan of $\text{m-PPOT}^{\alpha}(\bm{p}, \bm{q})$ can approximate that of $\text{PPOT}^{\alpha}(\bm{p}, \bm{q})$.
	
	\subsection{Can our model align the distributions of source and target common class data?}
	We present the value of $\textup{OT}(\bm{p}_c, \bm{q}_c)$ before and after the training of our model, as shown in Fig.~\ref{fig:dist}. Note that $\textup{OT}(\bm{p}_c, \bm{q}_c)$ decreases apparently after training, which empirically shows that our method is effective to align the distributions of source and target common class data.
	
	\subsection{Effectiveness of our model to minimize POT}
	We also provide the value of $\text{POT}^{\alpha}(\bm{p}, \bm{q})$ before and after training in Fig.~\ref{fig:dist} to testify whether $\text{POT}^{\alpha}(\bm{p}, \bm{q})$ can be minimized by our model. We find that the value of $\text{POT}^{\alpha}(\bm{p}, \bm{q})$ is apparently reduced after training, which empirically verifies the rationality of Theorem 1, because the $\mathcal{L}_{rce}$ and $\mathcal{L}_{ot}$ in our model are designed for minimizing the two terms in the upper bound of $\text{POT}^{\alpha}(\bm{p}, \bm{q})$ in Theorem 1.

	\subsection{More results for POT and PPOT}
	We have discussed the advantages of PPOT over POT in the ablation study (``Comparison of m-PPOT with m-POT'') of experiment section. Figure~\ref{fig:batch} presents the results of PPOT and POT with varying batch sizes for OPDA task C$\rightarrow$A. The figure shows that PPOT generally achieves higher performance in H-score than POT, and is more stable to the bath size than POT, especially when the batch size is less than 36. Moreover, for providing a more fair comparison with POT and PPOT, we train the network only with PPOT loss (10) (or POT loss) and a cross-entropy loss defined in source domain. Table \ref{tab:ablation study} shows their results for OPDA task in Office-31 and Office-Home datasets, we can see that PPOT performs well than POT in both datasets.

	\begin{figure}[!htbp]
		\centering 
		\includegraphics[width=0.3\textwidth]{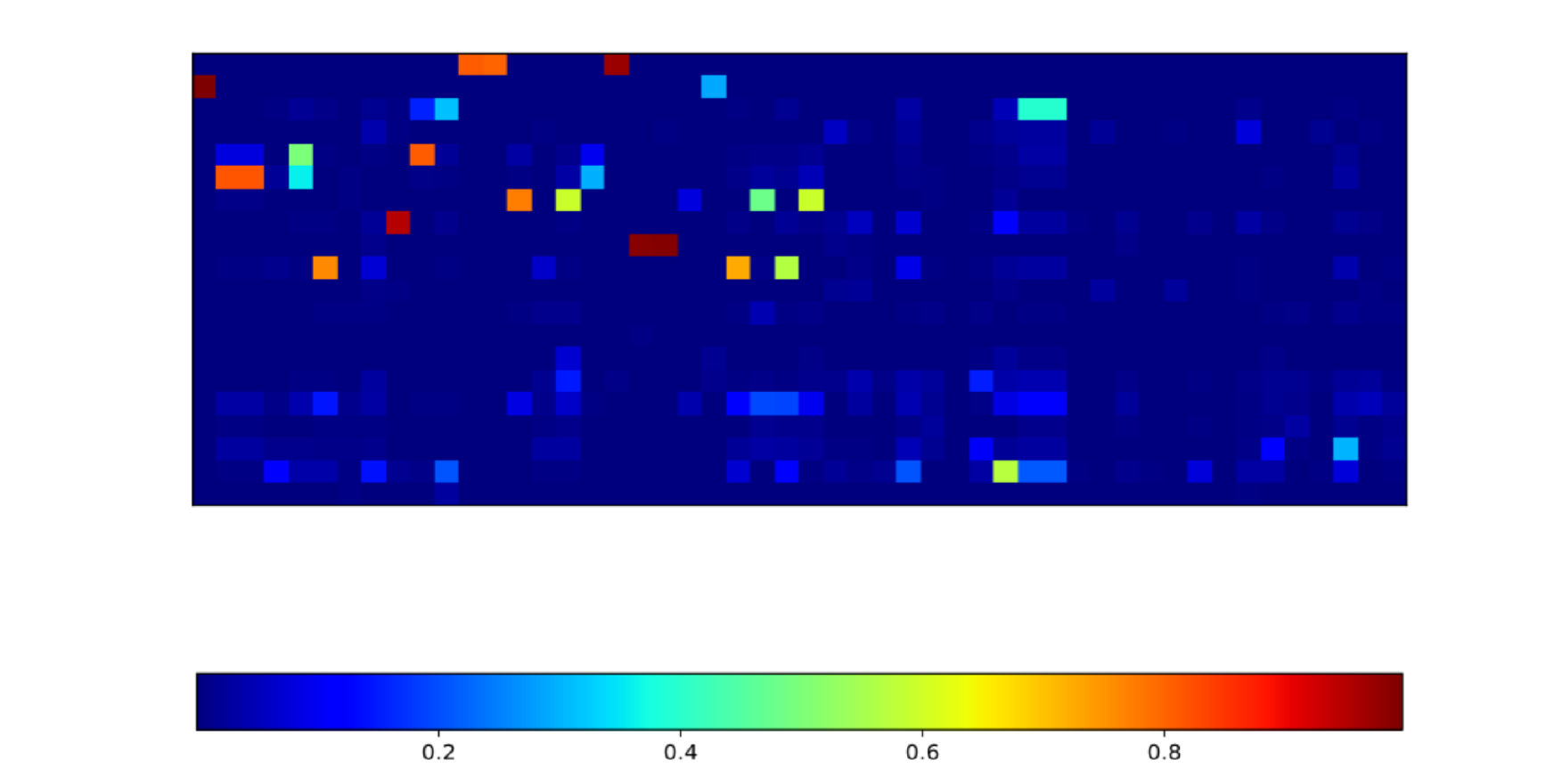} 
		\caption{Transport plan of PPOT in task W$\rightarrow$D on Office-31 datasets. Transport appears in the upper left of transport matrix is the correct transport.}
		\label{fig:realdata}
	\end{figure}
	
	\begin{figure}[!htbp]
		\centering 
		\includegraphics[width=0.3\textwidth]{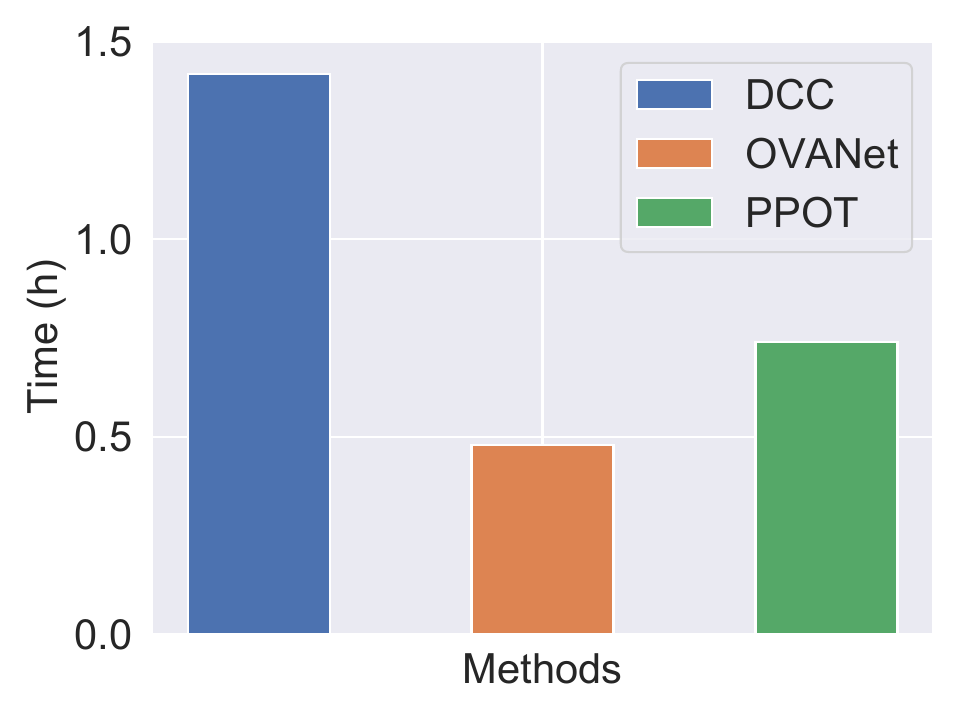} 
		\caption{Computational cost of UniDA methods in OPDA task A$\rightarrow$C.}
		\label{fig:cost}
	\end{figure}
	
	\subsection{More detail about prototype update strategy}
	We have illustrated that the set of prototypes $c$ is updated by exponential moving average in the paper, and here we show detail about this strategy. We update $c$ by the moving average:
	\begin{equation} \label{eq: c}
		c \leftarrow \lambda \hat{c} + (1 - \lambda) c
	\end{equation} 
	in each iteration, where $\hat{c}$ is computed by a batch of source features (note that if there is no sample with label $k$ in a batch, we denote $\hat{c}_k = c_k $ directly). We also report the sensitivity of hyper-parameter $\lambda$, the result is shown in Fig. \ref{fig:lambda}. Note that our model is relatively stable to varying values of $\lambda$, and the result decreases when we do not use moving average strategy (\ie, $\lambda = 1$).
	
	\subsection{Visualization of PPOT}
	We visualize the transport procedure of PPOT in toy data and the transport plan of PPOT in real data, as shown in Figs.~\ref{fig:toydata},~\ref{fig:realdata}. In the toy data experiment, each of the source data(blue) and target data (red) are sampled by Gaussian mixture distributions composed of three distinct Gaussian components indicated by different shapes where the same shapes indicate the same class (both source and target data have two common classes and one private class) and green points are the prototypes of source data. We can see that most of transport occur between target common class points and their corresponding source prototype. In real data experiment, we first compute the features of source prototypes and target samples and choose 50 target features randomly as target samples. The transport plan between source prototypes and target features is shown in Fig.~\ref{fig:realdata}, the task is W$\rightarrow$D on Office-31 datasets. Note that we rearrange target samples to ensure the indexes of common class samples are less than private class samples, which means correct transport will present in the upper left of transport plan.
	
	\subsection{Computational cost}
	We compare the computational cost of different methods with the total training time in the same training steps (5000 steps), as in Fig.~\ref{fig:cost}. Figure~\ref{fig:cost} shows that PPOT is comparable to other methods in terms of computational cost.

	\begin{table*}[!htbp]
		\centering
		\begin{tabular}{l|cccccccccccc|>{\columncolor{mygray}}c}
			\toprule
			~ & \multicolumn{12}{c}{Office-Home(10/5/50)} & ~\cellcolor{white}\\
			\midrule
			Methods&A2C&A2P&A2R&C2A&C2P&C2R&P2A&P2C&P2R&R2A&R2C&R2P&Avg\\
			\midrule
			UAN & 51.6 & 51.7 & 54.3 & 61.7 & 57.6 & 61.9 & 50.4 & 47.6 & 61.5 & 62.9 & 52.6 & 65.2 & 56.65\\
			CMU & 56.0 & 56.9 & 59.2 & 67.0 & 64.3 & 67.8 & 54.7 & 51.1 & 66.4 & 68.2 & 57.9 & 69.7 & 61.6\\
			DANCE & 61.0 & 60.4 & 64.9 & 65.7 & 58.8 & 61.8 & 73.1 & 61.2 & 66.6 & 67.7 & 62.4 & 63.7 & 63.9\\
			DCC & 58.0 & 54.1 & 58.0 & 74.6 & 70.6 & 77.5 & 64.3 & 73.6 & 74.9 & 81.0 & 75.1 & 80.4 &70.2\\
			OVANet & 63.4 & 77.8 & 79.7 & 69.5 & 70.6 & 76.4 & 73.5 & 61.4 & 80.6 & 76.5 & 64.3 & 78.9 & 72.7\\
			GATE & 63.8 & 75.9 & 81.4 & 74.0 & 72.1 & 79.8 & 74.7 & \bf 70.3 & 82.7 & 79.1 & \bf 71.5 & \bf 81.7 & 75.6\\
			\midrule
			\bf PPOT& \bf 66.0 & \bf 79.3 & \bf 84.8 & \bf 78.8 & \bf 78.0 & \bf 80.4 & \bf 82.0 & 62.0 & \bf 86.0 & \bf 82.3 & 65.0 & 80.8 & \bf 77.1\\
			\bottomrule
		\end{tabular}
		\caption{H-score(\%) comparison on Office-Home in OPDA setting.}
		\label{tab:OPDA_OH}
	\end{table*}
	
	\section{Detailed Results on Office-31 and Office-Home}
	We have reported the average accuracies over all tasks on Office-31 and Office-Home in the paper (see Tables 1, 2 and 3 of the paper). We report the detailed results for each task in this section.
	
	\subsection{Results for OPDA setting on Office-Home}
	We show the results of each task in Tables \ref{tab:OPDA_OF} and \ref{tab:OPDA_OH} on Office31 and Office-Home in OPDA setting. Our method achieves the best results in 13 out of 18 adaptation tasks.
	
	\subsection{Results for PDA and OSDA setting on Office-Home}
	Table \ref{tab:PDA_OH} shows that our method achieves the best results on average and surpasses state-of-the-art results on half of them on Office-Home in PDA setting. The results in Table \ref{tab:OSDA_OH} show that PPOT reaches the best results in 9 out of 12 tasks on Office-Home in OSDA setting.
	
	\begin{table*}[!htbp]
		\centering
		\begin{tabular}{l|c|cccccccccccc|>{\columncolor{mygray}}c}
			\toprule
			\multirow{2}*{Method} & \multirow{2}*{Type} & \multicolumn{12}{c}{Office-Home(25/40/0)} & ~\cellcolor{white}\\
			\cmidrule{3-15}
			~ & ~& A2C & A2P & A2R & C2A & C2P & C2R & P2A & P2C & P2R & R2A & R2C & R2P & Avg\\
			\midrule
			PADA & P & 52.0 & 67.0 & 78.7 & 52.2 & 53.8 & 59.1 & 52.6 & 43.2 & 78.8 & 73.7 & 56.6 & 77.1 & 62.1\\
			IWAN & P & 53.9 & 54.5 & 78.1 & 61.3 & 48.0 & 63.3 & 54.2 & 52.0 & 81.3 & 76.5 & 56.8 & 82.9 & 63.6\\
			ETN & P & 59.2 & 77.0 & 79.5 & 62.9 & 65.7 & 75.0 & 68.3 & 55.4 & 84.4 & 75.7 & 57.7 & 84.5 & 70.5\\
			AR & P & \bf65.7 & \bf87.4 & \bf89.6 & \bf79.3 & \bf75.0 & \bf87.0 & \bf80.8 & \bf65.8 & \bf90.6 & \bf80.8 & \bf65.2 & \bf86.1 & \bf79.4\\
			\midrule
			DCC & U & 54.2 & 47.5 & 57.5 & 83.8 & 71.6 & \bf86.2 & 63.7 & \bf65.0 & 75.2 & \bf85.5 & \bf78.2 & 82.6 & 70.9\\
			GATE& U & \bf 55.8 & 75.9 & 85.3 & 73.6 & 70.2 & 83.0 & 72.1 & 59.5 & \bf84.7 & 79.6 & 63.9 & 83.8 & 73.9\\
			\bf PPOT & U & 53.1 & \bf81.2 & \bf86.1 & \bf78.9 & \bf71.9 & 83.7 & \bf74.6 & 55.5 & 84.4 & 77.4 & 57.9 & \bf86.3 &  \bf74.3\\	
			\bottomrule
		\end{tabular}
		\caption{H-score(\%) comparison on Office-Home in PDA setting. ``P'' and ``U'' denote PDA and UniDA methods.}
		\label{tab:PDA_OH}
	\end{table*}

	\begin{table*}[!htbp]
		\centering
		\begin{tabular}{l|c|cccccccccccc|>{\columncolor{mygray}}c}
			\hline
			\multirow{2}*{Method} & \multirow{2}*{Type} & \multicolumn{12}{c}{Office-Home(25/0/40)} & ~\cellcolor{white}\\
			\cmidrule{3-15}
			& & A2C & A2P & A2R & C2A & C2P & C2R & P2A & P2C & P2R & R2A & R2C & R2P & Avg\\
			\midrule
			STA & O & 55.8 & 54.0 & 68.3 & 57.4 & 60.4 & 66.8 & 61.9 & 53.2 & 69.5 & 67.1 & 54.5 & 64.5 & 61.1\\
			OSBP & O &  55.1 & 65.2 & 72.9 & \bf64.3 & 64.7 & \bf70.6 & \bf63.2 & 53.2 & 73.9 & 66.7 & 54.5 & 72.3 & 64.7\\
			ROS & O & \bf 60.1 & \bf69.3 &\bf 76.5 & 58.9 & \bf65.2 & 68.6 & 60.6 & \bf56.3 & \bf74.4 & \bf68.8 & \bf60.4 & \bf75.7 & \bf 66.2\\
			\midrule
			DCC & U & 56.1 & 67.5 & 66.7 & 49.6 & 66.5 & 64.0 & 55.8 & 53.0 & 70.5 & 61.6 & 57.2 & 71.9 & 61.7\\
			OVANet & U & 58.9 & 66.0 & 70.4 & 62.2 & 65.7 & 67.8 & 60.0 & 52.6 & 69.7 & 68.2 & 59.1 & 67.6 & 64.0\\
			GATE & U & \bf 63.8 & 70.5 & 75.8 & 66.4 & 67.9 & 71.7 & 67.3 & \bf 61.5 & 76.0 & 70.4 & \bf 61.8 &  75.1 & 69.1\\
			\bf PPOT& U & 60.7 & \bf 75.2 & \bf 79.5 & \bf 67.3 & \bf 70.1 & \bf 73.8 & \bf 70.6 & 57.2 & \bf 76.1 & \bf 71.8 & 61.4 & \bf 75.8 & \bf 70.0\\
			\bottomrule
		\end{tabular}
		\caption{H-score(\%) comparison on Office-Home in OSDA setting. ``O'' and ``U'' denote OSDA and UniDA methods.}
		\label{tab:OSDA_OH}
	\end{table*}

\end{document}